\documentclass{article}

    \usepackage{arxiv}

\usepackage[utf8]{inputenc} 
\usepackage[T1]{fontenc}    
\usepackage{hyperref}       
\usepackage{url}            
\usepackage{booktabs}       
\usepackage{amsfonts}       
\usepackage{nicefrac}       
\usepackage{microtype}      
\usepackage{enumitem}
\usepackage{mathtools}
\usepackage{amsthm}
\usepackage{color, soul}
\usepackage[dvipsnames]{xcolor}
\usepackage{graphicx}
\usepackage{amsmath}
\usepackage{algorithm}
\usepackage{algpseudocode}
\usepackage{subcaption}
\usepackage{cleveref}

\DeclarePairedDelimiter{\norm}{\lVert}{\rVert}

\DeclarePairedDelimiter{\dott}{\langle}{\rangle}

\newcommand{\expec}{\mathbb{E}}
\newcommand{\E}{\mathbb{E}}
\newcommand{\w}{\mathbf{w}}
\newcommand{\bq}{\mathbf{q}}
\newcommand{\bk}{\pmb{\kappa}}
\newcommand{\F}{\mathcal{F}}

\newtheorem{theorem}{Theorem}

\newtheorem{lemma}{Lemma}
\newtheorem{Assumption}{Assumption}

\title{ByGARS: Byzantine SGD with Arbitrary Number of Attackers}

\author{%
   Jayanth R. Regatti\thanks{corresponding author} \\
   Department of ECE\\
   The Ohio State University\\
   Columbus, OH 43210 \\
   \texttt{regatti.1@osu.edu}
   \And
   Hao Chen \\
   Department of ECE \\
   The Ohio State University \\
   Columbus, OH 43210 \\
   \texttt{chen.6945@osu.edu}
   \And
   Abhishek Gupta \\
   Department of ECE \\
   The Ohio State University \\
   Columbus, OH 43210 \\
   \texttt{gupta.706@osu.edu}
}

\begin{document}

\maketitle

\begin{abstract}

We propose two novel stochastic gradient descent algorithms, ByGARS and ByGARS++, for distributed machine learning in the presence of any number of Byzantine adversaries. In these algorithms, reputation scores of workers are computed using an auxiliary dataset at the server. This reputation score is then used for aggregating the gradients for stochastic gradient descent. The computational complexity of ByGARS++ is the same as the usual distributed stochastic gradient descent method with only an additional inner product computation in every iteration. We show that using these reputation scores for gradient aggregation is robust to any number of multiplicative noise Byzantine adversaries and use two-timescale stochastic approximation theory to prove convergence for strongly convex loss functions. We demonstrate the effectiveness of the algorithms for non-convex learning problems using MNIST and CIFAR-10 datasets against almost all state-of-the-art Byzantine attacks. We also show that the proposed algorithms are robust to multiple different types of attacks at the same time.
\end{abstract}

\section{Introduction}
\label{sec:intro}

With increasing data size and model complexity, the preferred method for training machine learning models at scale is to use a distributed training setting. This involves a parameter server that coordinates the training with multiple worker machines by communicating gradients and parameters. Despite the speed up in computation due to the distributed setting, it suffers from issues such as straggling workers, while also posing a significant risk to the privacy if the data is collected at a central location. Federated Learning \cite{konevcny2016federated, bonawitz2019towards} addresses  these issues where the central server only has access to the model parameters/gradients computed by the workers (or independent data owners). However, both the settings are prone to fail in the presence of dishonest workers or non-malicious failed workers \cite{kairouz2019advances}.

Thus, there has been a significant interest in devising distributed machine learning schemes in the presence of Byzantine adversaries \cite{alistarh2018byzantine, chen2017distributed, blanchard2017machine, gupta2019byzantine}. A certain fraction of the workers are assumed to be adversarial; instead of sending the actual gradients computed using a randomly sampled mini batch to the server, the adversarial workers send arbitrary or potentially adversarial gradients that could derail the optimization at the server. Several techniques have been proposed to secure the gradient aggregation against adversarial attacks under different settings such as gradient encoding \cite{chen2018draco}, asynchronous updates \cite{damaskinos2018asynchronous, xie2019zeno++}, heterogeneous datasets \cite{li2019rsa}, decentralized learning \cite{yang2019byrdie}, \cite{yang2019adversary, el2019sgd} and federated Learning \cite{chang2019cronus, portnoy2020towards}. There has also been some work in developing attack techniques that break existing defenses \cite{chang2019cronus, xie2019fall, baruch2019little}.

\begin{figure}[!b]
    \centering
    \vspace*{-1em}
    \label{fig:place1}
    \includegraphics[width=0.45\textwidth]{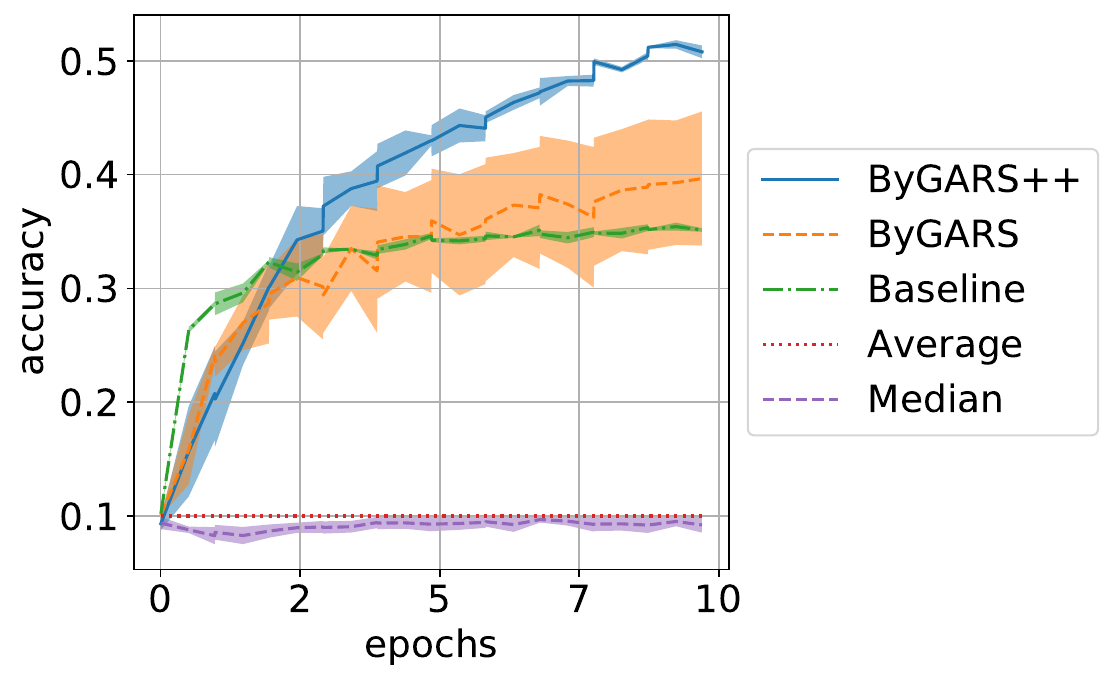}
    \vspace*{-1em}
    \caption {Comparison of the top-1 accuracy of ByGARS and ByGARS++ using CIFAR-10 dataset with one benign worker and seven attackers using different attack strategies.  The seven attackers include one Gaussian adversary, two Sign flip adversaries, one random sign flip adversary, two label flip adversaries, and one constant value adversary. See Section~\ref{sec:simulations} for more details.}

    \label{fig:my_label}
\end{figure}

\begin{table*}[t]
\caption{Summary of various attacks that ByGARS or  ByGARS++ is robust to. Extensive simulations suggests that the proposed algorithms are resilient to most of the state-of-the-art attacks with any number of Byzantine adversaries; the checkmarks in the right column indicates that in simulations, we have found either ByGARS or ByGARS++ is able to perform SGD under a wide range of initial conditions. Here, $f$ denotes the fraction of Byzantine adversaries in the system, with $f=1$ implying that all workers are adversarial.}
\label{table:attack summary}
\vskip 0.15in
\begin{center}
    \begin{tabular}{|c| c | ccc |}
        \hline
           &  & \multicolumn{3}{c}{Fraction of Adversaries $f$} \vline\\
        Type & Attack & $f< 0.5$ & $f\in[0.5,1)$ & $f = 1$ \\
        \hline
        Omniscient / Collusion & Inner Product Manipulation \cite{xie2019fall} & $\surd$  & $\surd$  & \\
        \hline    
        Omniscient / Collusion & LIE \cite{baruch2019little} & $\surd$ & -  & \\
        \hline    
        Omniscient / Collusion & OFOM \cite{chang2019cronus} & $\surd$  & $\surd$  &  \\
        \hline    
        Omniscient / Collusion & PAF \cite{chang2019cronus}  & $\surd$  & $\surd$   & \\
        \hline    
        Local / Failure & Sign Flip/Reverse Attack & $\surd$  & $\surd$  & $\surd$  \\
        \hline    
        Local / Failure & Random Sign Flip Attack  & $\surd$  & $\surd$ & $\surd$  \\
        \hline    
        Local / Failure & Gaussian Attack \cite{blanchard2017machine} & $\surd$  & $\surd$   & \\
        \hline    
        Local / Failure & Constant Attack \cite{li2019rsa} & $\surd$  & $\surd$ & \\
        \hline    
        Data Poisoning & Label Flipping & $\surd$  & $\surd$  & \\
        \hline    
        \textit{Mixed Attacks} & Multiple types of attacks  & $\surd$  & $\surd$ & $\surd$ \\
        \hline    
    \end{tabular}%
\end{center}
\end{table*}

One of the main assumptions in past studies about Byzantine attacks in machine learning is that the number of adversarial workers is less than half of the total number of workers. These approaches relied on techniques like majority voting, geometric median, median of means, coordinate wise median, etc. to aggregate gradients at the server. The fundamental reason for this assumption was that the underlying concept of geometric median (often used for robust aggregation) has a breakdown point of 0.5 \cite{chen2017distributed}. In other words, it yields a robust estimator as long as less than half of the data (used for aggregation) is corrupted.

The assumption that less than half of the workers are adversarial might not be practical. A more challenging problem is to ensure convergence even in the presence of a large number of adversaries. Some prior works that address this case are \cite{xie2018zeno, jin2019distributed, cao2019distributed, xie2019zeno++}, among a few others. In these works, the server has some auxiliary data, which is used to identify adversarial workers and the gradients obtained from such workers are discarded at the server. In contrast, we address the issue of an arbitrary number of adversaries by allowing the server to use an auxiliary dataset (a small dataset drawn from the same distribution as the training data) to compute the {\it reputation score} of each worker that is used for gradient aggregation. The {\it reputation score} of a worker signifies how relevant the corresponding gradient direction is to the optimization problem. As we will see, we achieve robustness to any number of adversaries by letting the reputation score of the workers take negative values.

The key insight that allows us to train under any number of adversaries is that (a) the gradients are being computed for a specific objective function, and (b) the correct gradients would make a small angle with respect to the gradient computed using the auxiliary data with high probability (when the current parameters are far from the optimal ones). To see this, note that geometric median is a robust aggregator under any objective function. Thus, it necessarily requires a more stringent assumption on the number of adversaries to compute a reliable estimate. We explicitly use an unbiased estimate of the objective function (through the auxiliary dataset) in our algorithm, which eases the computation and allows us to not use general purpose robust aggregator like geometric median. Further, since the auxiliary dataset has (roughly) the same distribution as the original dataset, we expect that two stochastic gradients (computed on small mini-batches of the two datasets) would make a small angle with each other with high probability. On the other hand, a random vector will be almost orthogonal to the correct stochastic gradients due to the intrinsic properties of random high dimensional vectors. This intuition allows us to compute the reputation score for each worker reliably early on in the training process. Indeed, we show that our algorithm enjoys convergence under reasonable assumptions on the objective function and assuming multiplicative noise adversaries. Empirical evidence suggests that our algorithm is robust to a large class of Byzantine attacks (summarized in Table \ref{table:attack summary}), not just the ones with multiplicative noise adversaries. 

\textbf{Our Contributions:} In this work, we do not identify the adversarial workers and discard their gradients; instead, we use the auxiliary data at the server to compute a reputation score for the workers, and use the reputation scores to weigh the gradients of the workers to carry out the parameter update. Our primary contributions are: 
\begin{enumerate}[leftmargin=0.5cm]\itemsep -0.2em
    \item We propose a novel reputation score based gradient aggregation method for distributed machine learning with learnable reputation scores which may be of independent interest
    \item We show that our algorithm is Byzantine tolerant (in the sense of \cite{xie2019fall}) to an arbitrary number of attackers
    \item We use two time-scale stochastic approximation theory to establish the convergence of the proposed algorithm under reasonable assumptions (with strongly convex loss function)
    \item Empirical evidence on convex and non-convex objectives suggests that our proposed algorithms are robust to almost all state-of-the-art Byzantine attacks. We also show that our algorithms can defend mixed attacks (where multiple different attacks are performed at the same time). To the best of our knowledge, we are the first work to demonstrate such ability
    \item Previous works that depend on filtering out the adversaries required an estimate on the number of adversaries (often not available in practice), and thus the reduced effective data size for training affected the test error and the generalization ability of the trained model. Our proposed algorithms, on the other hand, do not require such assumptions and achieve better generalization as observed from simulations
    %
\end{enumerate}

\section{Related Work}
\label{sec:related_work}

\textbf{Byzantine SGD (fraction of adversaries $<$ 0.5):} Blanchard et al. \cite{blanchard2017machine} showed that the then existing approaches of aggregation of gradients using linear combinations were inadequate to deal with Byzantine adversaries. They showed experimentally that the presence of a single Byzantine adversary is enough to derail the training. They proposed an algorithm $\textit{Krum}$ and its variants and proved the convergence for non convex functions. Mhamdi et al. \cite{mhamdi2018hidden} proposed $L_p$ norm attacks that break Krum, and further proposed an algorithm (Bulyan) based on meta selectio. However, Bulyan is robust to less than $0.25$ fraction of adversaries. Chen et al. \cite{chen2017distributed} proposed a median of means algorithm to robustly aggregate the gradients, and showed rigorous convergence guarantees for strongly convex objective functions. Similarly, Yin et al. \cite{yin2018Byzantine} studied theoretical guarantees of coordinate wise median, and trimmed mean approaches. Alistarh et al.  Following the idea of coordinate-wise median, Yang et al. \cite{yang2019byzantine} use a Lipschitz inspired coordinate median algorithm to filter out the adversaries. Bernstein et al. \cite{bernstein2018signSGD} propose a simple, computationally cheap and efficient technique to use only the element wise sign of the gradients and perform a majority voting, but it's applicability against a wide range of attacks is unclear. Despite strong guarantees, the above median and majority inspired approaches fail to apply to an arbitrary number of attackers. \cite{alistarh2018byzantine} show optimal sample complexity for convex problems by using historical information of the gradients sent by workers even in high dimensions. Despite considering the historical information, this algorithm does not apply to an arbitrary number of attackers. 

\textbf{Byzantine SGD (arbitrary number of adversaries):}  By assuming availability of auxiliary data Cao \& Lai \cite{cao2019distributed} filter out Byzantine adversaries, and show convergence guarantees for strongly convex objective functions.  Under similar assumptions, Xie et al. \cite{xie2018zeno} propose $\textit{Zeno}$. They formulate a stochastic descent score similar to the Armijo condition for line search, to filter out adversarial workers, and prove convergence guarantees for non-convex objectives. A similar stochastic descent score is used in \cite{xie2019zeno++} under asynchronous settings. Jin et al. \cite{jin2019distributed} also study the case of arbitrary number of adversaries in an asynchronous setting. Although \cite{jin2019distributed} does not assume an auxiliary dataset, the algorithm requires each worker to have access to the gradients of all other workers, and using their own data the workers filter out the adversaries and use them benign workers for local aggregation and parameter update. Note that filtering out the gradients requires an estimate on the number of adversaries which leads to reduction in effective data size and loss of any useful information in the gradients (for example system failures). In contrast to these earlier works, instead of filtering out the adversaries, we assign a {\it reputation score} to each work which is used in aggregation. Ji et al. \cite{ji2019learning} also use an auxiliary dataset to aggregate gradients using an LSTM. In comparison, we use a simple reputation score based aggregation, with theoretical guarantees. Yao et al. \cite{yao2019federated} also compute an auxiliary gradient and use it for the update in the Federated Learning setting. However, the reputation scores of their algorithm are constant (fraction of samples) and don't address adversaries. To the best of our knowledge, our algorithm is the first one to use the idea of {\it reputation scores} to efficiently aggregate gradients, and show robustness to a wide range of attacks (also including multiple attacks by different workers at the same time).

\section{Problem setup}\label{sec:problem}

We consider distributed machine learning with a parameter server - worker setup. The parameter server maintains the model parameters, and updates the parameters with gradients received from the workers. We denote the model parameters by $\w\in\mathcal{W}\subset\mathbb{R}^d$, and the number of workers by $m$. We assume that each worker $j$ has access to dataset, $D_j :=\{x_i^j, y_i^j\}_{i=1}^{n_j} \sim \mathcal{D}$, where $N=\sum_{j}n_j$ is the total number of data points. In the Federated Learning scenario, this translates to each worker having its own dataset, which is not shared with anyone. In the distributed machine learning scenario, the server assigns the data partitions to the workers uniformly at random. Given a loss function $f(\cdot,x,y): \mathbb{R}^d \to \mathbb{R}$, $x,y \sim \mathcal{D}$, the objective is to minimize the population loss $F:\mathbb{R}^d\to \mathbb{R}$

\begin{equation}
\label{eq:poploss}
    \w_* = \arg\min_{\w \in \mathbb{R}^d} F(\w) := \expec_{x,y \sim \mathcal{D}}[f(\mathbf{w},x,y)]
\end{equation}

We denote the true gradient of the population loss at $\w_t$ by $\nabla F(\w_t)$. A good worker samples a subset of the data $\mathcal{D}_{j,t}\subset \mathcal{D}_j$, and computes a stochastic gradient     $\Tilde{h}_{t,j}:= \frac{1}{|\mathcal{D}_{j,t}|}\sum_{x,y\in \mathcal{D}_{j,t}}\nabla f(\w_t, x, y)$. The good workers communicate the stochastic gradient $h_{t,j}:=\tilde{h}_{t,j}$ to the server, where as adversarial workers inject a random multiplicative noise $\tilde{\kappa}_i$ and send $h_{t,j}:=\tilde{\kappa}_i \tilde{h}_{t,j}$, where $\tilde{\kappa}_i$ is a random variable that the adversary draws at each time step from a fixed attack distribution. We assume that this attack distribution remains fixed for the adversary throughout the training. We denote the set of gradients received by the server as $H_t^T = [h_{t,1}, \cdots, h_{t,m}] \in \mathbb{R}^{d\times m}$. Note that we assume a synchronous setting here, i.e. all the workers communicate the gradients at the same time to the server. We assume that the server has access to an auxiliary dataset $D_{aux}:= \{x_i,y_i\}_{i=1}^n \sim \mathcal{D}$. The server can sample a subset $\xi_{aux,t}$ of the auxiliary dataset and compute auxiliary loss $L_t(\w) = \frac{1}{|\xi_{aux,t}|}\sum_{(x,y)\in \xi_{aux,t}} f(\mathbf{w},x,y)$, such that $\E [\nabla L_t(\w_t)] = \nabla F(\w_t)$.

\section{Algorithm}

In this section, we motivate the importance of using a reputation score for gradient aggregation. In an ideal environment, where all the workers are benign, the gradient aggregation function simply averages the received stochastic gradients and uses the averaged gradient to update the parameters. If the batch sizes used by the workers are different, then a weighted averaging of the stochastic gradients is performed. However, our problem setup is far from ideal, it involves an arbitrary number of workers that act as Byzantine adversaries (potentially all of them can be adversarial). 

To compute a meaningful estimate of the gradient, the server maintains a reputation score $q_{t,j}$ for each worker $j$. Since the adversary can be of any type, the reputation scores can take any real value. Suppose, at time $t$, the reputation score vector is $\bq_t = [q_{t,1},\ldots,q_{t,m}]^T$ and the received gradients are $H_t$, then the weighted aggregation of the gradients with the reputation score is $H_t^T\bq_t = \sum_{i=1}^{m}q_{t,i}h_{t,i}$. If we have a good estimate of the reputation score $\bq_t$ (say we know $\boldsymbol \kappa$ and set $q_{t,i} = \frac{1}{\kappa_i}$ for all $i$) at each time $t$, then $-H_t^T\bq_t$ is a descent direction, then no adversary can affect the training provided that sufficiently small step size is used and the adversaries satisfy certain assumptions.

The problem now is to compute a good reputation score, without the knowledge of $\boldsymbol\kappa$, for the workers using only the gradients sent to the server. Making use of our two key assumptions -- availability of an auxiliary dataset and the stationary behavior of the workers, we propose two algorithms to compute the reputation score of the workers. 

\subsection{ByGARS}

\begin{algorithm}[t]
    \centering
    \caption{ByGARS: Byzantine Gradient Aggregation using Reputation Scores}\label{alg:bygars}
    \begin{algorithmic}[1]
        \State $\w_0$ \text{initialized randomly and sent to workers} 
        \State $\bq_0 = \pmb{0}$
        \For{$t=1,\cdots,T$}
            \State \text{receive} $H_t^T = [h_{t,1}, \cdots, h_{t,m}]$ from workers
            \State $\bq_{t+1}^0 = \bq_t$
             \For{$i=1,\cdots,k$}
                \State ${\widehat{\mathbf{w}}}_{t+1} \leftarrow \mathbf{w}_t -\gamma_t H_t^T \bq_{t+1}^{i-1}~:= \w_t - \gamma_t \sum_{j=1}^{m}q_{t,j}^{i-1}h_{t,j}$ \Comment{perform pseudo update}
                \State $\bq_{t+1}^{i} \leftarrow \bq_{t+1}^{i-1} + \alpha_t \gamma_t H_t\nabla L_t(\widehat{\mathbf{w}}_{t+1})$ \Comment{perform meta update}
            \EndFor
            \State $\bq_{t+1}=\bq_{t+1}^{k} $ 
            \State $\w_{t+1}\leftarrow \w_{t} -\gamma_t H_t^T\bq_{t+1}$\Comment{perform actual update on parameters}
            \State Send $\w_{t+1}$ to workers
        \EndFor
        \State \textbf{Return} $\w_{T+1}$
    \end{algorithmic}
\end{algorithm}

We start with an initial reputation score of $\bq_0=\pmb{0}\in\mathbb{R}^m$, and iteratively improve the estimate of the reputation score. At step $t$, we perform a \textit{pseudo update} to $\w_t$ ($\gamma_t$ is a step size parameter) as 
\begin{equation}
    \label{eqn:bygars_pseudoupdate}
    \hat{\w}_{t+1} \leftarrow \w_t - \gamma_t H_t^T\bq_t 
\end{equation}
If $\bq_t$ is a good reputation score and $\gamma_t$ is sufficiently small, then $-H_t^T\bq_t$ is a descent direction and thus $F(\hat{\w}_{t+1})$ must be lower in value than $F(\w_t)$ or other points in its neighborhood. However, we neither have access to the true function $F$ nor the data from the workers. Instead, we have a small auxiliary dataset that is drawn from the same distribution as the data at the workers. This auxiliary dataset allows us to construct the loss function $L_t(\cdot)$ (see Section \ref{sec:problem}), and we can solve the following optimization problem to compute a better reputation score $\bq_{t+1}^*$:
\begin{equation}\label{eqn:metaopt}
    \bq_{t+1} = \arg\min_{\bq\in\mathbb{R}^m} L_t(\w_t - \gamma_t H_t^T\bq)
\end{equation}

Using the current estimate $\bq_t$, we use an iterative update rule. We compute the loss on a random mini-batch of the auxiliary dataset $D_{aux}$ using $\hat{\w_t}$, which is denoted as $L_t(\hat{\w_t})=L_t(\w_t - \gamma_t H_t^T\bq_t)$, and henceforth referred to as auxiliary loss. The objective of the \textit{meta update} is to minimize this loss with respect to $\bq$, given $H_t, \w_t$. We do this by performing a first order update to $\bq_t$ by computing the gradient of the auxiliary loss evaluated at $\hat{\w}_t$ wrt $\bq_t$. We refer to this gradient as the auxiliary gradient denoted by $\nabla L_t (\hat{\w}_t)$.  The gradient computation and \textit{meta update} to $\bq_t$ ($\alpha_t$ is a step size parameter) is given by

\begin{align}\label{eqn:metaupdate}
    \begin{split}
        \bq_{t} \leftarrow~& \bq_{t} - \alpha_t \frac{d}{d\bq_t} L_t(\w_t - \gamma_t H_t^T\bq_t)\\
        &= \bq_t - \alpha_t (-\gamma_t H_t) \nabla L_t(\w_t -\gamma_t H_t^T\bq_t) \\
        &= \bq_t + \alpha_t \gamma_t H_t \nabla L_t(\hat{\w_t})
    \end{split}
\end{align}

The updated reputation score is used to find the updated gradient aggregation $H_t^T\bq_t$. The algorithm proceeds by successively applying the \textit{pseudo update} (eq \ref{eqn:bygars_pseudoupdate}) and \textit{meta update} to $\bq_t$ (eq \ref{eqn:metaupdate}) for $k$ iterations (or until a stopping criteria is reached, such as sufficient descent) to obtain $\bq_{t+1}$ before finally performing an actual update to $\w_t$ as
\begin{equation}
    \label{eqn:bygars_actualupdate}
    \w_{t+1} \leftarrow \w_t - \gamma_t H_t^T\bq_{t+1} 
\end{equation}
This is summarized in Algorithm~\ref{alg:bygars} (please note the change in notation, eg. superscript $i$, due to the meta updates).

The core of our algorithm lies at the auxiliary dataset available to the worker. It is a reasonable assumption to make since in practical scenarios, it is not difficult to procure a small amount of clean auxiliary data without violating the privacy of the worker's data. This data can be taken from publicly available datasets (that match the distribution of the data available at the workers), from prior data leaks (that is now publicly available), or data given voluntarily by workers. In supplementary material, we provide more analysis on the affect of auxiliary dataset size on the performance of our algorithms.

\subsection{ByGARS++: Faster ByGARS}

Algorithm~\ref{alg:bygars} has an additional computational overhead due to multiple parameter updates and multiple gradient computations to update the reputation score in the meta updates. This increased computation at the server keeps the workers idle and waiting, thus negating the computational speed-up achieved from distributed learning. In order to overcome this limitation, and driven by the motivation of ByGARS, we propose a variant which is computationally cheaper yet efficient.

\begin{algorithm}[t]
    \centering
    \caption{ByGARS++}\label{alg:lite}
    \begin{algorithmic}[1]
        \State $\w_0$ \text{initialized randomly and sent to workers} 
        \State $\bq_0 = \pmb{0}$
        \For{$t=1,\cdots,T$}
            \State \text{Receive} $H_t^T = [h_{t,1}, \cdots, h_{t,m}]$ from workers
            \State \text{Compute $\nabla L_t(\w_t)$ using a subset of the auxiliary data}
            \State $\w_{t+1} \leftarrow \w_t - \gamma_t H_t^T \bq_t$ 
            \State Send $\w_{t+1}$ to workers
            \State $\bq_{t+1} \leftarrow (1-\alpha_t)\bq_{t} + \alpha_t H_t\nabla  L_t(\mathbf{w}_{t})$ \Comment{Computed when workers compute gradients}
        \EndFor
        \State \textbf{Return} $\w_{T+1}$
    \vfill
    \end{algorithmic}

\end{algorithm}

We propose Algorithm \ref{alg:lite} (ByGARS++), in which we avoid computing multiple \textit{pseudo updates} $\hat{\w}_t$ used for performing \textit{meta updates}, by simulataneously updating $\w_t, \bq_t$ as given by eq (\ref{eqn:updatelite}). Note that we perform an update to $\bq_t$ using the auxiliary gradients evaluated at $\w_t$ (and not at $\hat{\w}_t$). 

\begin{align}\label{eqn:updatelite}
    \begin{split}
        \w_{t+1} &\leftarrow \w_t -\gamma_t H_t^T\bq_t, \\ \bq_{t+1} &\leftarrow (1-\alpha_t)\bq_t + \alpha_t H_t \nabla L_t(\w_t)
    \end{split}
\end{align}

In this case, the reputation score of each worker is updated using only the inner product between the gradient sent by the worker, and the auxiliary gradient, both evaluated at $\w_t$. The only additional computation as compared to traditional distributed SGD is the udpate of $\bq_t$ which takes $\mathcal{O}(md)$ time. However, the server can update $\bq_t$ when the workers are computing the gradients for next time step (line 7, 8 of Algorithm~\ref{alg:lite}), therefore ByGARS++ has the same computational complexity as traditional distributed SGD.

If a worker consistently sends the gradient scaled with a negative value, then the reputation score accumulates negative values, since the inner product is negative in expectation. Therefore by multiplying the received gradient with $q_{t,j}$ we can recover the actual direction. When the parameters are far away from the optima, we compute the reputation score of each stochastic gradient by taking an inner product with the stochastic gradient computed on the auxiliary data. When the parameters are closer to the optima, the inner product value is random \cite{chee2018convergence} (since the directions of the stochastic gradients are random), and hence does not contribute to the reputation score. If unheeded, this phenomenon can destroy the reputation of good workers/boosts that of adversaries, therefore we employ a decaying learning rate schedule for both $\gamma_t$ and $\alpha_t$. Thus, by the time the parameters are close enough to the optima or a flat region (in non-convex settings), the learning rates would have decayed significantly. This enables the reputation score to accumulate over time and converge; therefore, the score is robust to the noisy inner products near the optima.

\section{Convergence of ByGARS++}
We now analyze the convergence of ByGARS++, which follows the update equation (\ref{eqn:updatelite}).


\begin{Assumption}\label{ass:strongly}
The population loss $F$ is $c$-strongly convex,  with $\w_*$ as the unique global minimum, such that $\nabla F(\w_*)=0$. Further, $\nabla F$ is a locally Lipschitz function with bounded gradients.
\end{Assumption}

\begin{Assumption}
The Byzantine adversaries corrupt the gradients using multiplicative noise. If the worker $i$ computes a stochastic gradient $\tilde{h}_{t,i}$ which is an unbiased estimate of $\nabla F(\w_t)$, the worker sends $h_{t,i}:=\tilde{\kappa}_{t,i} \tilde{h}_{t,i}$ to the parameter server, where $\tilde{\kappa}_{t,i}$ is an iid multiplicative noise with mean $\kappa_i$ and finite second moment. The random noise satisfies  $|\tilde{\kappa}_{t,i}|\leq \kappa_{\max}$ almost surely for all the workers. The workers have the following types:\vspace{-.5em}
\begin{enumerate}[leftmargin=0.5cm]\itemsep -0.2em
    \item Benign worker: $\E h_{t,i} = \nabla F(\w_t)$ with $\kappa_i=1$;
    \item Scaled adversary: $\E h_{t,i} = \kappa_i \nabla F(\w_t)$, where $\kappa_i$ is a real number (negative or positive);
    \item Random adversary: $\E h_t = 0$,where adversary sends random gradients with mean 0 (i.e. $\kappa_i = 0$).  
\end{enumerate}
There is at least one benign or scaled adversary with $\kappa_i\neq 0$ among the workers. Further, we assume the adversaries' noise distributions do not change with time.
\end{Assumption}
 This is a reasonable assumption as the goal of the attacker is to derail the training progress by corrupting the aggregate gradient so that it is not a descent direction. This also includes system failures, where the sign bit of the gradient is flipped erroneously during communication \cite{xie2018zeno}. This adversary model is used in several works including \cite{bernstein2018signSGD,li2019rsa,xie2019zeno++}. However, we show empirical results for different types of attacks (summarized in Table~\ref{table:attack summary}).
 
The following theorems capture the main result of this paper.
\begin{theorem}\label{thm:byzresilient}
If Assumption 2 is satisfied, then ByGARS++ is DSSGD-Byzantine Tolerant (Def. 4 in \cite{xie2019fall}), that is, $\E[\dott{\nabla F(\w_t), H^Tq_t}] \geq 0$.
\end{theorem}
\begin{proof}
The proof follows by applying the law of iterated expectations. We refer the reader to the appendix for a detailed proof.
\end{proof}
Thus, by accruing reputation scores using the inner products for every worker, the sign of $q_{t,i}$ is the same as that of $\kappa_i$. Due to this reason, even gradients of some adversaries can be helpful in training if multiplied with an appropriate reputation score. When the loss function is strongly convex and smooth, we can prove an even stronger result, which is captured below.
\begin{theorem}\label{thm:main}
Suppose that $\{\alpha_t\},\{\gamma_t\}$ are diminishing stepsizes, that is, $\sum\alpha_t=\infty, \sum\gamma_t=\infty,\sum\alpha_t^2<\infty,\sum\gamma_t^2<\infty$, with $\gamma_t/\alpha_t\rightarrow 0$ as $t\rightarrow\infty$. If Assumptions 1 and 2 are satisfied and $\sup_{t}\|\w_t\|,\sup_t\|\bq_t\|<\infty$, then $\{\w_t\}$ generated by \text{ByGARS++} converges almost surely $\w_*$.
\end{theorem}
The proof leverages two timescale stochastic approximation to establish the above result. The assumption of $\sup_{t}\|\w_t\|,\sup_t\|\bq_t\|<\infty$ is typical in stochastic approximation literature; see, for instance, \cite{borkar1997stochastic,borkar2009stochastic,tadic2004almost}. To avoid making this assumption, a typical workaround is to project the iterates back to a very large set in case the iterates go outside the set \cite{kushner2003stochastic,borkar2009stochastic}. For single timescale stochastic approximation, \cite{borkar2000ode} derives some sufficient conditions under which the iterates would be automatically bounded almost surely. However, we are unable to leverage this result since this method has not been extended to two timescale stochastic approximation. We now prove Theorem \ref{thm:main}.


\paragraph{Proof of Theorem \ref{thm:main}}
We use $\pmb{\kappa} = [\kappa_1,\ldots,\kappa_m]^T$ to denote the mean vector of the random multiplicative noises of the workers. We show that all the hypotheses of Theorem 2 of \cite{tadic2004almost} is satisfied by ByGARS++, which leads to the desired convergence result. Consider the two functions and the corresponding differential equations (here $\w,\bq$ are functions of continuous time and $\dot{\w}(t)$ denotes differentiation with respect to time $t$):
\begin{align}
    G_1(\w,\bq) &= - ( \bk^T \bq ) \nabla F(\w),\quad \dot{\w}(t) = G_1(\w(t),\bar \bq(t)) \label{eqn:g1w}\\ 
    G_2(\w,\bq) &= -\bq + \bk \|\nabla F(\w)\|^2, \quad \dot{\bq}(t) = G_2(\w,\bq(t)).\label{eqn:g2q}
\end{align}


\begin{lemma}\label{lem:diffeqn}
The differential equation in \eqref{eqn:g2q} has a unique globally asymptotically stable equilibrium, which is denoted by $\phi(\w)$ and is given by $\phi(\w) = \bk \|\nabla F(\w)\|^2$. The differential equation in \eqref{eqn:g1w} with $\bar\bq(t) = \phi(\w(t)) $ has a unique globally asymptotically stable equilibrium $\w_*$.
\end{lemma}
\begin{proof}

To see the first result, not that for any $\w$, the solution to the differential equation is $\bq(t) = (\bq(0)-\bk \|\nabla F(\w)\|^2)\exp(-t)+\bk \|\nabla F(\w)\|^2$, which is globally asymptotically stable with $\phi(\w) = \bq(\infty) = \bk \|\nabla F(\w)\|^2$. We now use Lyapunov stability theory to establish the second statement. Let us substitute into \eqref{eqn:g1w} $\bar\bq(t) = \phi(\w(t))$. Now define the Lyapunov function $V(\w) := F(\w) - F(\w_*)$, which is a valid Lyapunov function since $F$ is strongly convex by Assumption 1. We get
\begin{align*}
    &\nabla V(\w(t))^T G_1(\w(t),\phi(\w(t))) =  -\|\bk\|^2\|\nabla F(\w(t))\|^4<0 \text{ for all } \w(t)\neq \w_*.
\end{align*}
Thus, the differential equation has a unique globally asymptotically stable equilibrium where the Lyapunov function is 0, that is, at $\w_*$.
\end{proof}
Define $u_t = -G_1(\w_t,\bq_t) - H_t^T \bq_t$ and $v_t = H_t \nabla L_t(\w_{t})-\bk \norm{\nabla F(\w_t)}^2$. Using the definition of $u_t$ and $v_t$, the algorithm ByGARS++ is rewritten as
\begin{align}
    \begin{split}
    \w_{t+1} &= \w_t -\gamma_t H_t^T \bq_t \\
    &= \w_t + \gamma_t G_1(\w_t,\bq_t) + \gamma_t u_t
    \end{split}\\
    \begin{split}
    \bq_{t+1} &= (1-\alpha_t) \bq_{t} + \alpha_t H_t \nabla L_t(\w_{t})\\
    &= \bq_t + \alpha_t G_2(\w_t,\bq_t) + \alpha_t v_t.
    \end{split}
\end{align}
We show that $u_t$ and $v_t$ are martingale difference stochastic processes. Let $(\Omega,\mathcal F,\mathbb{P})$ be a standard probability space and let $\mathcal F_t$ be the $\sigma$-algebra generated by all the randomness realized up to time $t$:
\begin{align*}
    \begin{split}
        \mathcal F_t = \sigma\{ &\w_0,h_{1,0},\ldots,h_{m,0},\bq_0,\ldots,\w_{t-1},h_{1,t-1}, \ldots,h_{m,t-1},\bq_{t-1},\w_t,\bq_t\}.
    \end{split}
\end{align*}
It is easy to see that $\mathcal F_t\subset\mathcal F_{t+1}$, and thus, $\{\mathcal F_t\}_{t\in\mathbb{N}}$ is a filtration. It is straightforward now to observe that $\E [H_t|\F_t]=\bk \nabla F(\w_t)^T$, which implies that $\E[u_t|\F_t] = 0$. Further, it is easy to deduce that $\E[\|u_t\|^2|\F_t]\leq M(1+\|\bq_t\|^2)$ for a large $M>0$ that is dependent on the bounds on $\|\nabla F(\w_t)\|$ and $\|\kappa\|$. Thus, $u_t$ is a martingale difference noise. Next, observe that $H_t, \nabla L_t(\w_t)$ are independent given $\w_t, \bq_t$. Therefore $\E[H_t\nabla L_t(\w_t)|\F_t] = \E[H_t|\F_t]\E[\nabla L_t(\w_t)|\F_t] = \bk \norm{\nabla F(\w_t)}^2$. Again, it is easy to show that $\E[\|v_t\|^2|\F_t]\leq M(1+\|\bq_t\|^2)$. This implies $\{v_t\}$ is also a martingale difference noise.

It is clear from the expressions that since $\nabla F$ is locally Lipschitz, $G_1$ and $G_2$ are locally Lipschitz maps. In addition, $\phi(\w)$ is also a locally Lipschitz map. Theorem \ref{thm:main} now follows from the result in Lemma \ref{lem:diffeqn} and Theorem 2 of \cite{tadic2004almost}.


\section{Simulations}
\label{sec:simulations}

In this section, we discuss the experimental setup used to evaluate the proposed algorithms. 

\subsection{Dataset and models} 
\label{sec:datasets}

We present the results of our algorithms on MNIST \cite{lecun2010mnist}, CIFAR-10 \cite{krizhevsky2009learning} for multi-class classification using supervised learning. We used a LeNet architecture \cite{lecun1998gradient} for MNIST, and a two convolutional layer CNN for CIFAR-10. For each dataset, we set aside a small auxiliary dataset of size 250 (sampled randomly from the training data) at the server, and the remaining training data is distributed uniformly to the workers. More details on the hyperparameters, models, size of the auxiliary dataset, and results on a synthetic dataset are provided in the supplementary material.

\subsection{Attack mechanisms} 
\label{sec:attack_mechanisms}

The summary of the attacks used in this work is given in Table~\ref{table:attack summary}. The attacks were grouped broadly into (i) Omniscient / Collusion attacks, (ii) Local attacks / System failures, and (iii) Data Poisoning attacks. We further propose a \textit{mixed attack}, where we combine multiple Local attacks and Data Poisoning attacks. 

In the \textbf{Omniscient / Collusion} attacks, the adversaries have complete information about all other workers including the benign ones, or only about the other adversaries. In LIE attack \cite{baruch2019little}, the adversary adds well crafted perturbations to the empirical mean of the benign gradients that is sufficient to avoid convergence. In OFOM, PAF \cite{chang2019cronus} a large arbitrary vector is added to the empirical mean of the benign gradients and sent to the server. In Inner Product Manipulation, Xie et al. \cite{xie2019fall} multiply the empirical mean of benign gradients with a negative value.

In \textbf{Local} attacks, the attacker doesn't have any information about the other workers. Instead, the worker either sends an arbitrary gradient to the server or uses the gradient it computed. Examples of the first case include a Gaussian Attack (a vector drawn from a Gaussian distribution of mean 0, covariance 200) \cite{blanchard2017machine}, or Constant attack (a vector of all 1s multiplied by an arbitrary scalar, say 100) \cite{li2019rsa}. In the second case, the attacker computes a gradient on locally computed data, and sends a negatively scaled value of the gradient (this is termed as reverse attack, sign flipping attack in various works) \cite{li2019rsa,bernstein2018signSGD,jin2019distributed}. Hardware/communication failures that corrupt the gradients (unintentionally) by flipping the sign bit can also be included under these attacks. In addition to the sign flipping attack, we propose a {\it random} sign flip attack, where at each iteration the adversary draws a real number from a fixed distribution (optionally scales it with a large constant), multiplies with the local gradient and sends to the server. Note that the mean of the distribution can take negative as well as positive values. In our simulations we used a Gaussian distribution with mean randomly picked around -2 and variance 1.

In \textbf{Data Poisoning} attacks, the underlying data used to compute the gradients is poisoned so that the model outputs the attacker chosen targets during inference \cite{baruch2019little}. There are several varieties of data poisoning (also called backdooring) attacks, but we only consider label flipping attack in this work. Under the label flipping attack, for example in MNIST dataset, the attacker maps the labels as $l \to (9 - l)$ for $l \in \{0, \cdots, 9\}$, and uses these labels for computing the gradients. Note that, label flipping attack is also a Local attack. 
 
In addition to these attacks, we propose a \textit{Mixed Attack} where there can be multiple Local attacks and Data Poisoning attacks. This is motivated by the need to develop algorithms that are robust to several different types of attacks at the same time.

\subsection{Baselines}
\label{sec:baselines}

The generalization performance of SGD improves with the size of the dataset \cite{hardt2015train}. Since existing techniques for arbitrary number of Byzantines filter out the gradients that are sent by the adversarial workers, (in a true Federated Setting) the generalization performance of those techniques is limited by the number of benign workers and the amount of data available with them. In the case where all workers are adversaries, the only available truthful data is the auxiliary dataset. Hence, we consider plain averaging of the gradients available at these benign workers (and auxiliary gradient) as the \textit{Baseline}. Note that this is the \textbf{best} any Byzantine resilient algorithm that relies on filtering out adversarial gradients can do. Primarily for this reason, along with other implementation specific details required for other works that consider an arbitrary number of adversaries (Zeno requires trim parameter $b$ that needs knowledge of the number of adversaries), we chose to compare our algorithm with \textit{Baseline} (described above) as the gold standard. In addition to this, for illustration purposes, we also consider plain averaging of all gradients (no defense) denoted by \textit{Average}, and coordinate-wise median \cite{yin2018Byzantine} denoted by \textit{Median} in our empirical analysis.

\subsection{Implementation Details}
\label{sec:implementation}

Throughout this section, we will assume a setup with 1 server and 8 workers. We first compare the performance of the algorithms in the absence of any attacks (termed as \textit{No Attack}). In order to show the robustness of our proposed algorithms to an arbitrary number of attackers, we consider multiple cases with different number of attackers, i.e. 3, 6, 8. We consider the case of all 8 adversaries for Sign Flipping attack. As we will see, by allowing negative reputation scores for these workers, our algorithms will achieve similar performance as that of \textit{No Attack}. 

In order to avoid ~\texttt{NaN, Inf} values (due to adversaries), we normalize all the gradients (irrespective of adversarial or benign or auxiliary) during training. Specifically, we normalize the auxiliary gradients always to have unit norm, for ByGARS all worker gradients are normalized to unit norm, for ByGARS++ all worker gradients are normalized to value 2, and for the rest (\textit{Baseline, Average, Median}) all worker gradients are normalized to value 5.


\begin{figure*}[!tbh]
    \centering
    \subfloat[]{
        \label{fig:cifar_noattack}
        \includegraphics[width=0.45\textwidth]{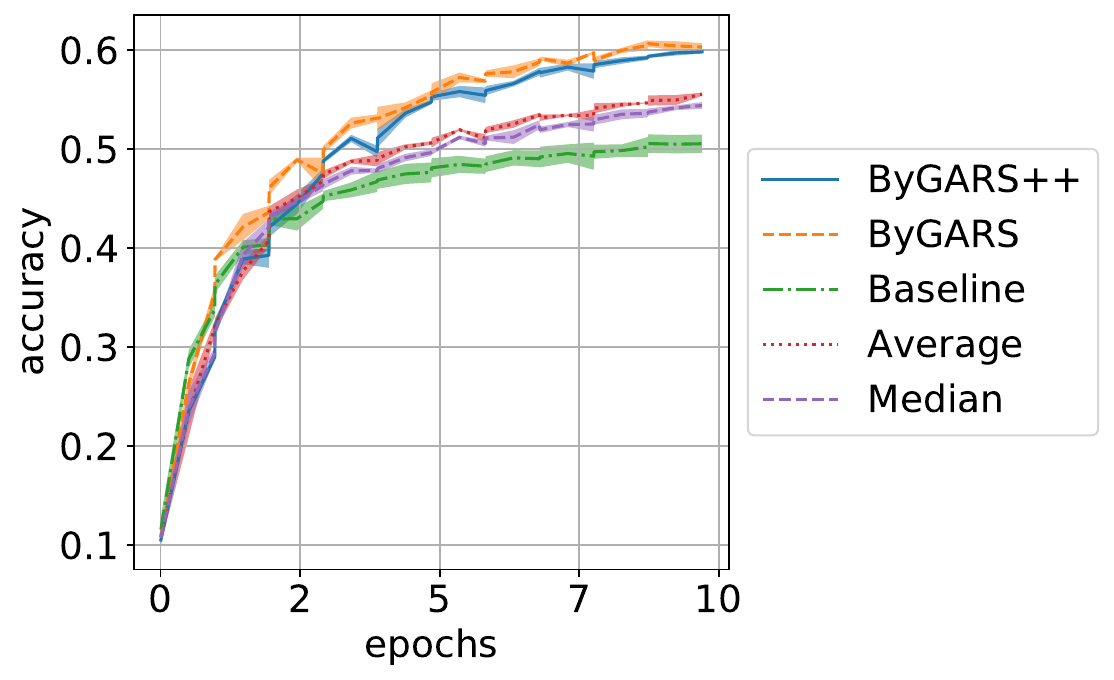}
    }
    \subfloat[]{
        \label{fig:cifar_signflip_5}
        \includegraphics[width=0.45\textwidth]{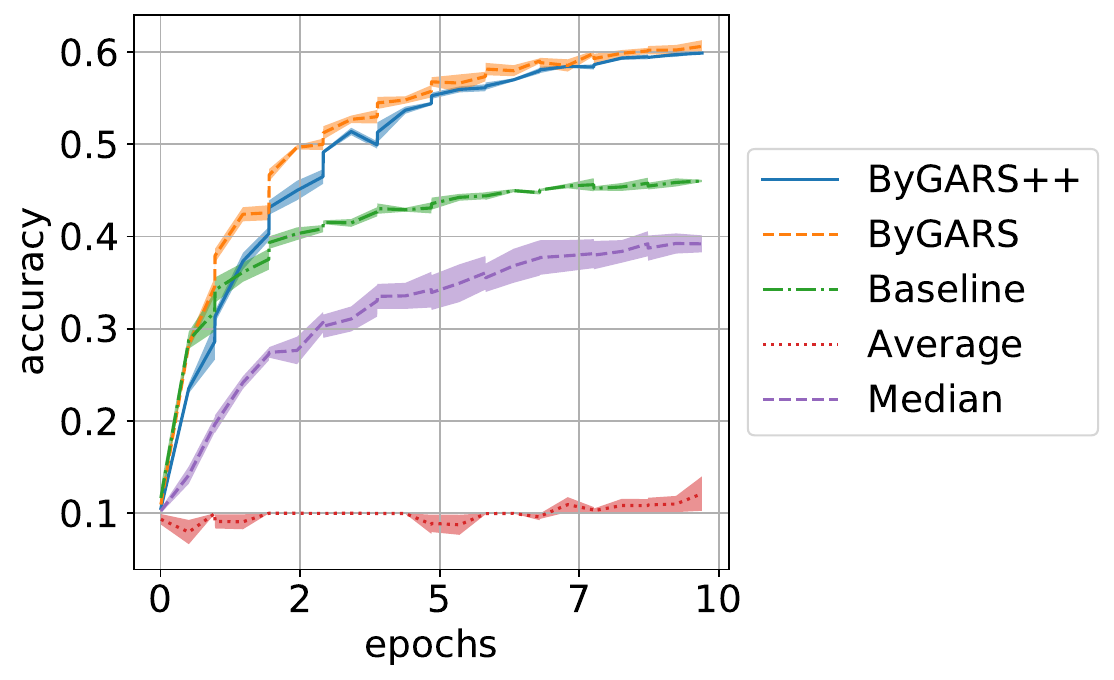}
    }\\
    \subfloat[]{
        \label{fig:cifar_signflip_2}
        \includegraphics[width=0.45\textwidth]{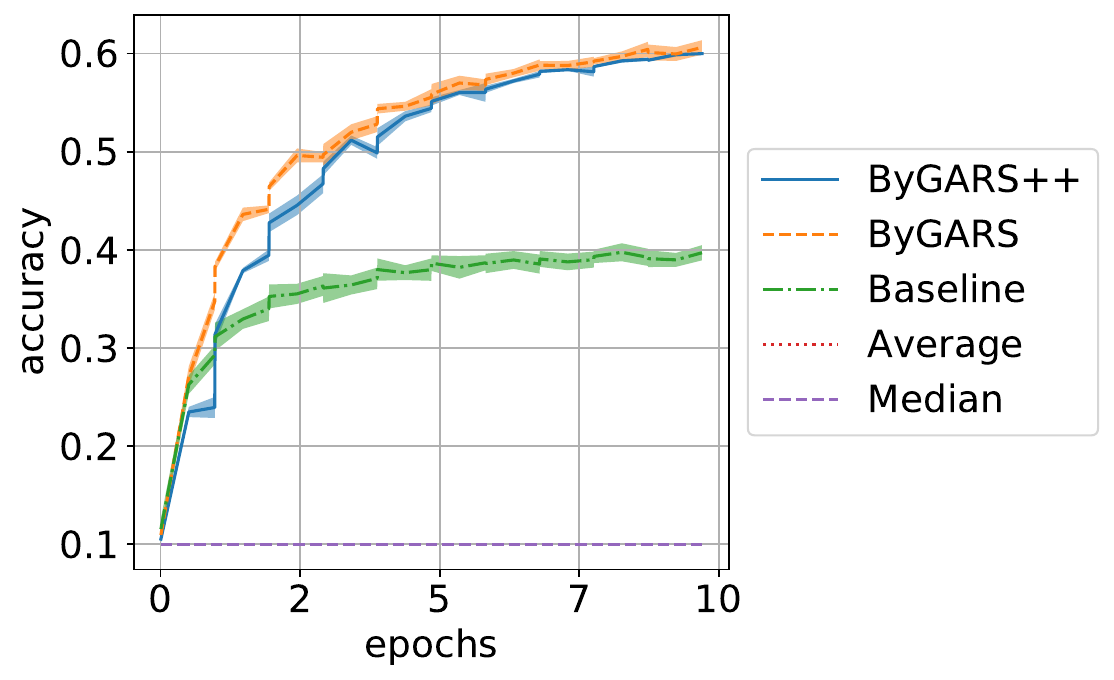}
    }
    \subfloat[]{
        \label{fig:cifar_signflip_0}
        \includegraphics[width=0.45\textwidth]{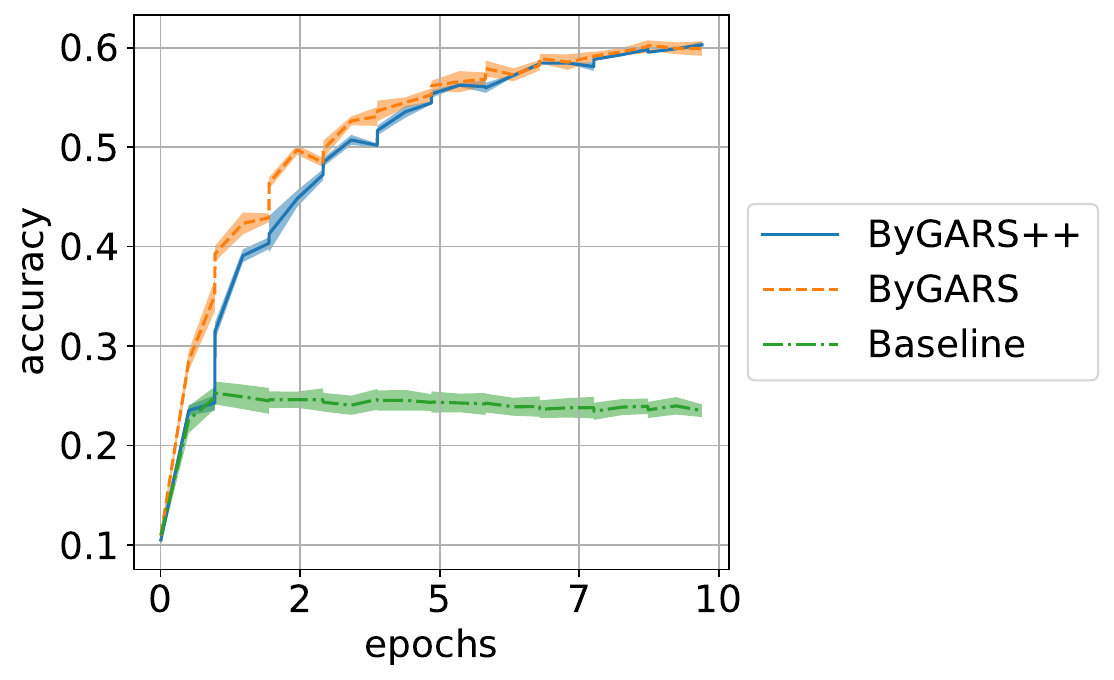}
    }
    \caption{\label{fig:cifarmain} This figure shows the top-1 accuracy of models trained on CIFAR-10 dataset under  (\ref{fig:cifar_noattack}) \textit{No Attack}; (\ref{fig:cifar_signflip_5}) 3 Sign flip adversaries; (\ref{fig:cifar_signflip_2}) 6 Sign flip adversaries;  (\ref{fig:cifar_signflip_0}) 8 Sign flip adversaries. Note that, despite the presence of 100\% adversaries, we are able to recover the performance of \textit{No Attack}. Also note that the only truthful gradients  available to the baseline in \ref{fig:cifar_signflip_0} is from the auxiliary data, and hence the worse performance.}
\end{figure*}

\begin{figure*}[!tbh]
    \centering
    \subfloat[]{
        \label{fig:cifar_innerprod_5}
        \includegraphics[width=0.32\textwidth]{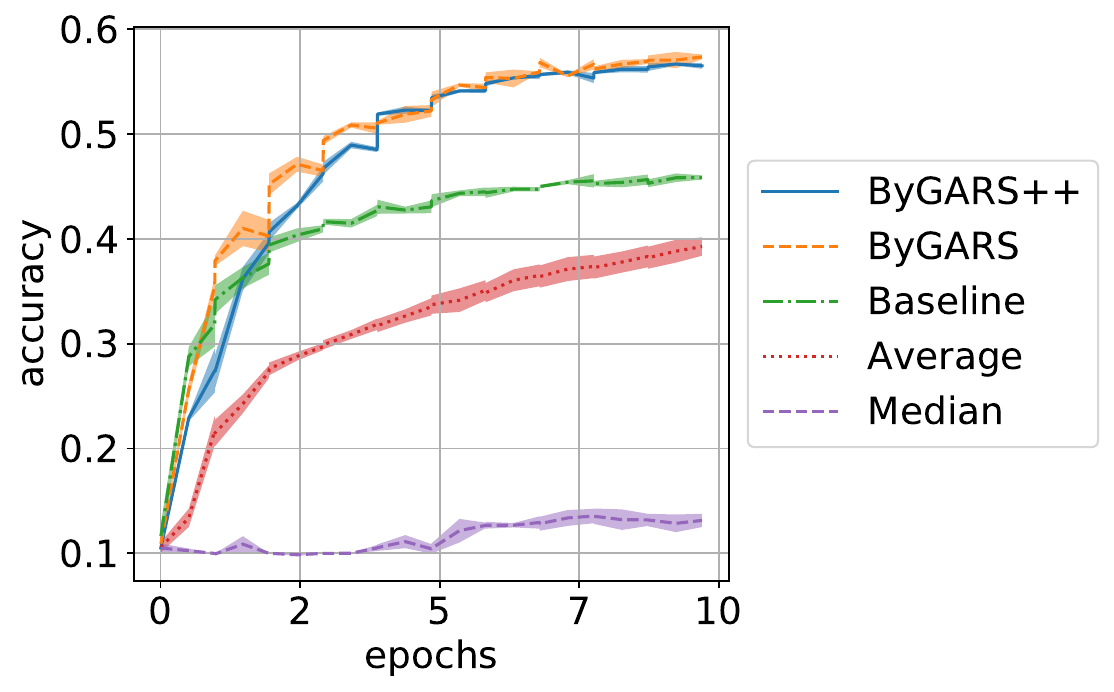}
    }
    \subfloat[]{
        \label{fig:cifar_labelflip_5}
        \includegraphics[width=0.32\textwidth]{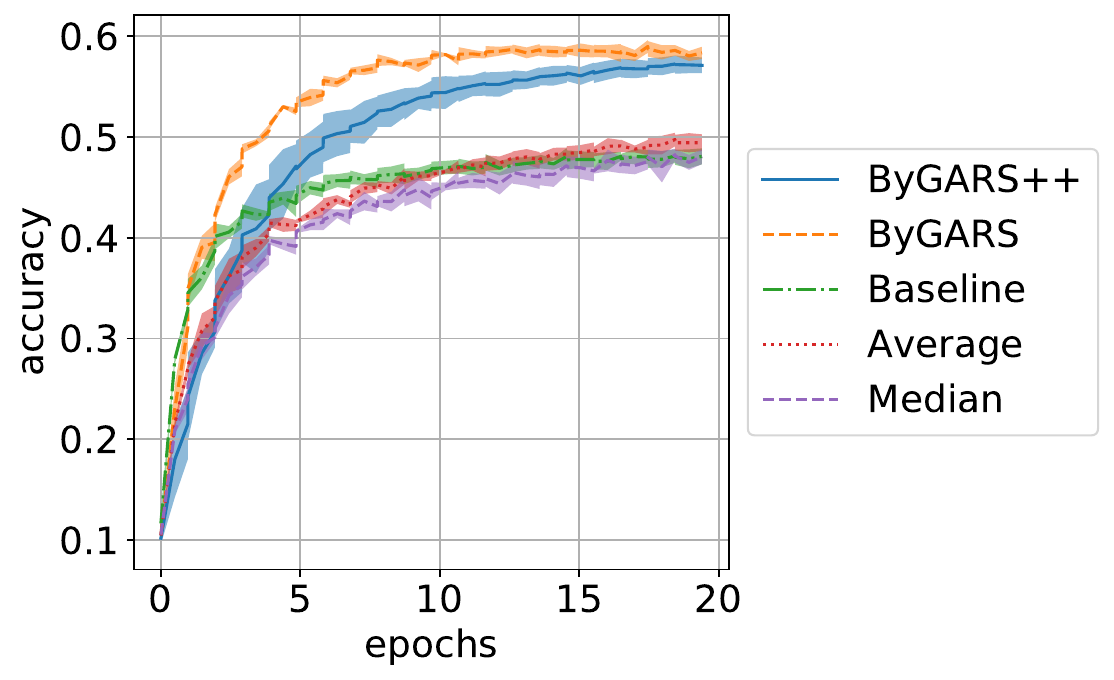}
    }
    \subfloat[]{
        \label{fig:cifar_const_5}
        \includegraphics[width=0.32\textwidth]{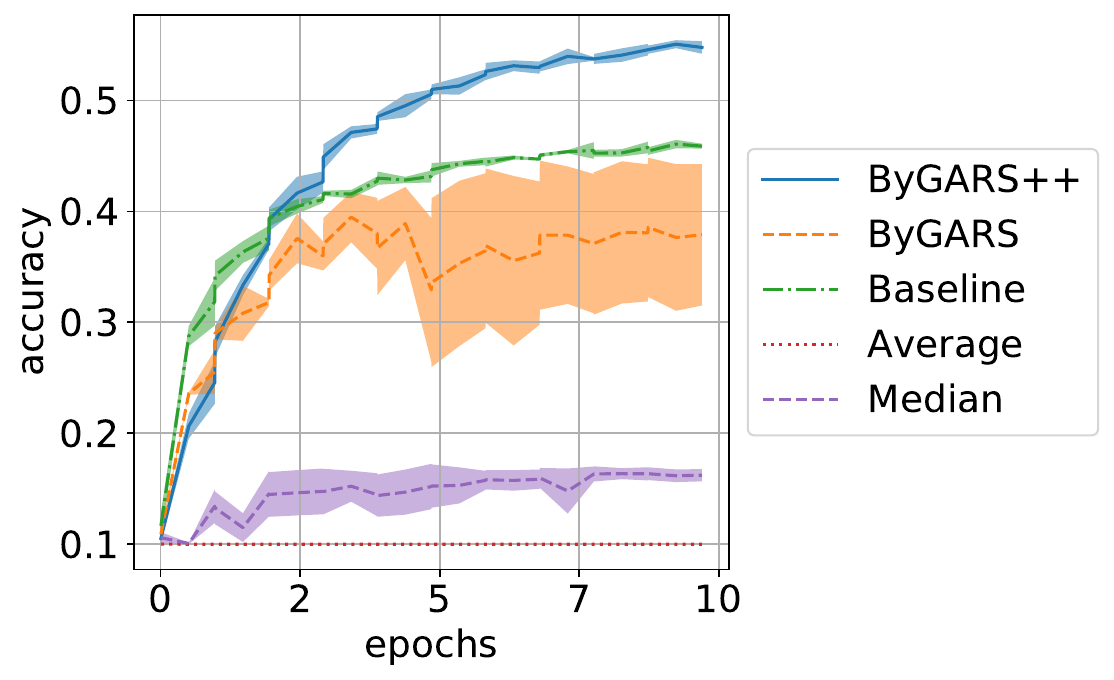}
    }\\
    \subfloat[]{
        \label{fig:cifar_innerprod_2}
        \includegraphics[width=0.32\textwidth]{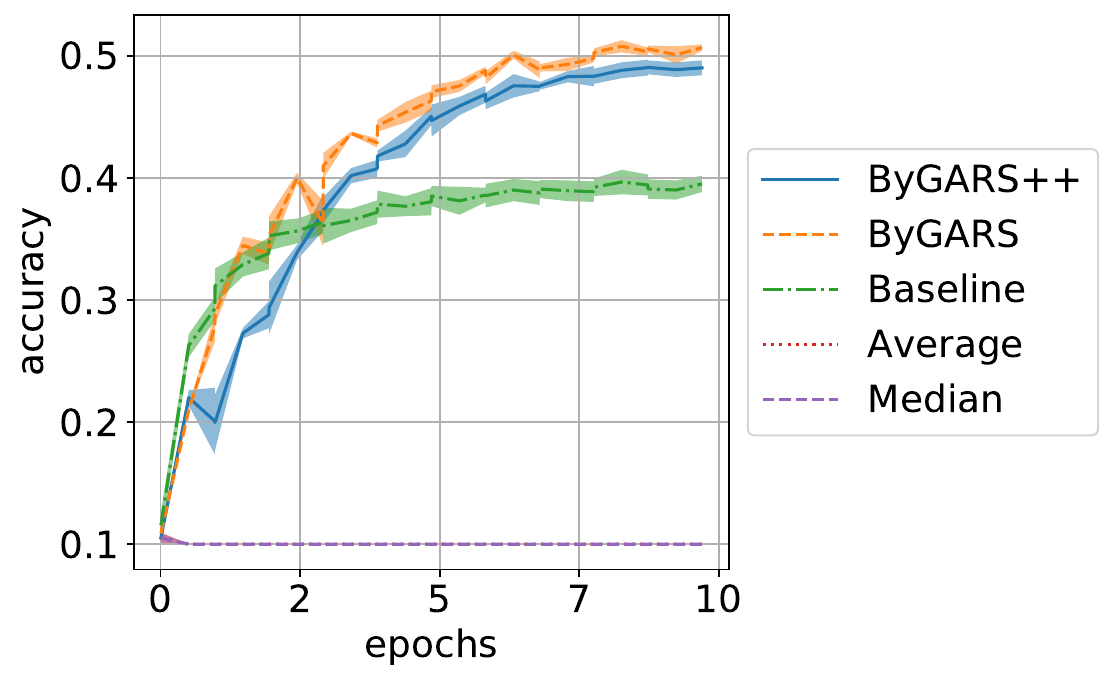}
    }
    \subfloat[]{
        \label{fig:cifar_labelflip_2}
        \includegraphics[width=0.32\textwidth]{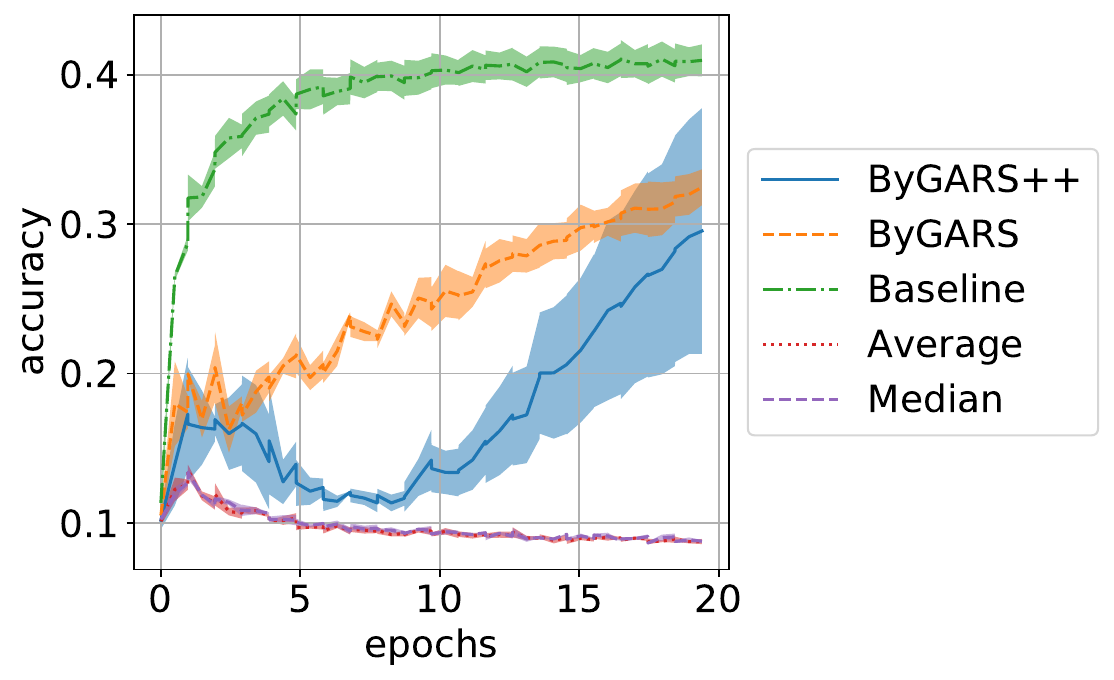}
    }
    \subfloat[]{
        \label{fig:cifar_const_2}
        \includegraphics[width=0.32\textwidth]{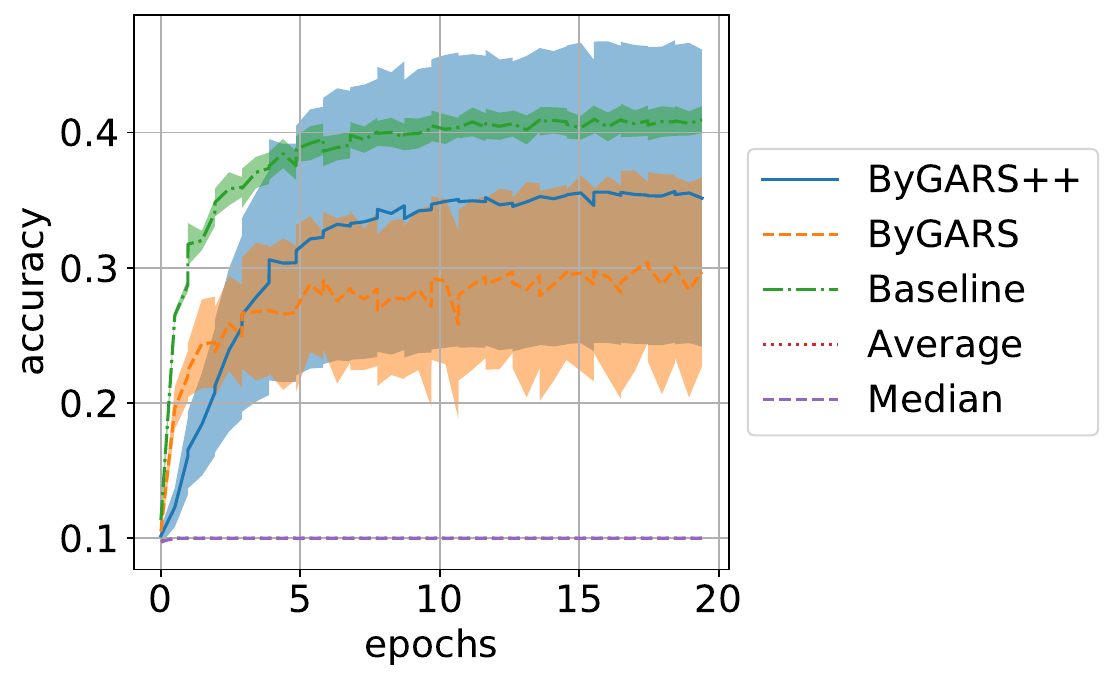}
    }
    \caption{Top-1 accuracy for models trained on CIFAR-10 data under (\ref{fig:cifar_innerprod_5}) 3 Inner Product adversaries, (\ref{fig:cifar_innerprod_2}) 6 Inner Product adversaries, (\ref{fig:cifar_labelflip_5}) 3 Label Flip adversaries, (\ref{fig:cifar_labelflip_2}) 6 Label Flip adversaries, (\ref{fig:cifar_const_5}) 3 Constant Attack adversaries, and (\ref{fig:cifar_const_2}) 6 Constant Attack adversaries. Note the reduction in performance from the top row to the bottom row, due to the increase in the number of adversaries. Despite that, ByGARS, ByGARS++ perform reasonably well.}
    \label{fig:cifarplots_2}
\end{figure*}

\begin{figure*}[!tbh]
    \centering
    \subfloat[]{
        \label{fig:cifar_randomsignflip_5}
        \includegraphics[width=0.32\textwidth]{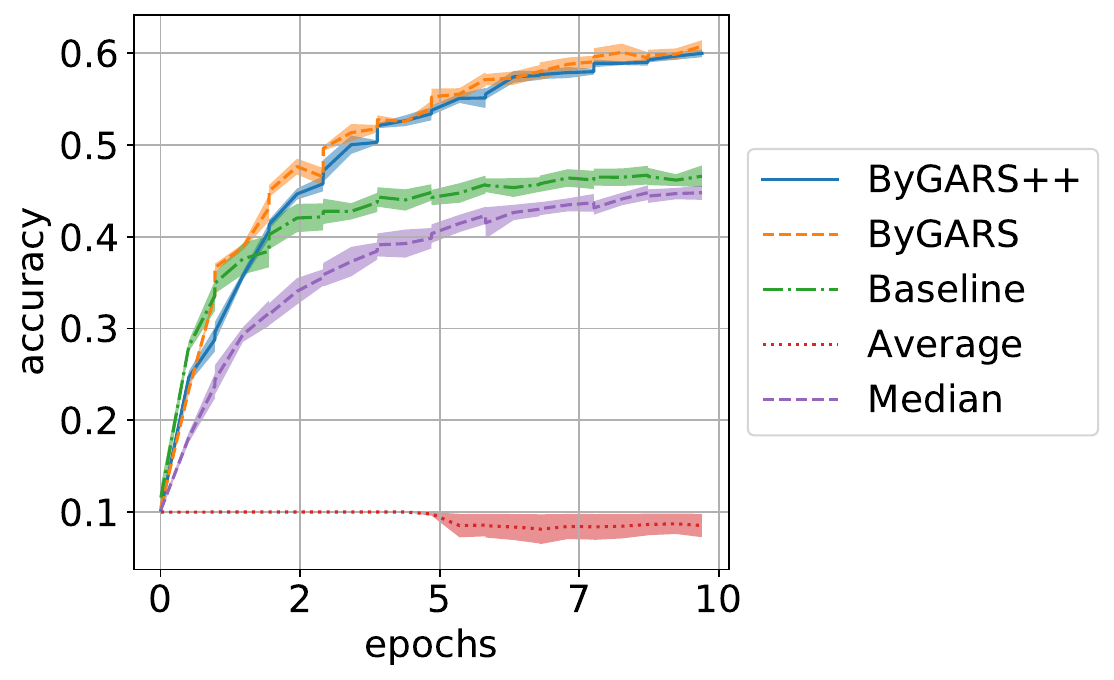}
    }
    \subfloat[]{
        \label{fig:cifar_ofom_5}
        \includegraphics[width=0.32\textwidth]{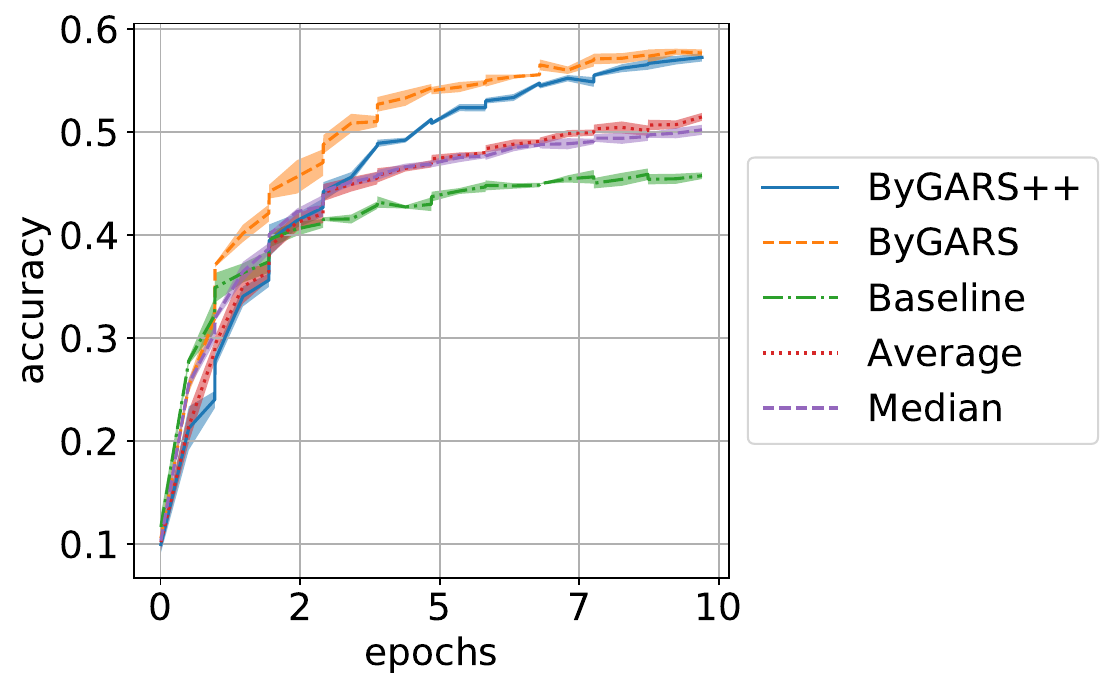}
    }
    \subfloat[]{
        \label{fig:cifar_lie_5}
        \includegraphics[width=0.32\textwidth]{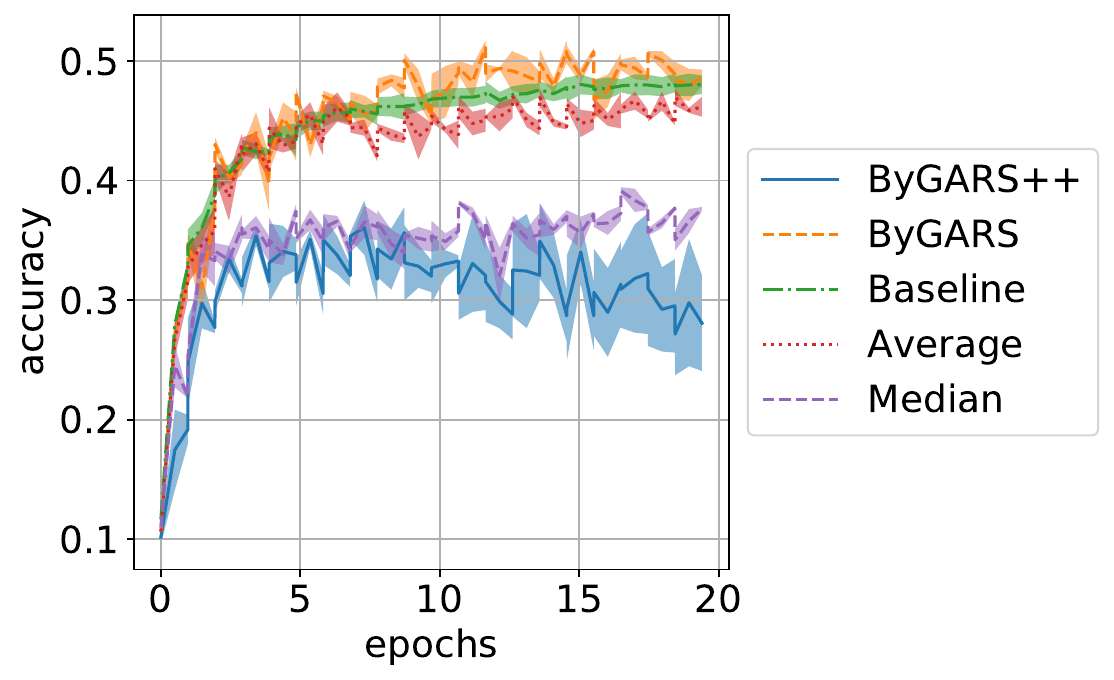}
    }\\
    \subfloat[]{
        \label{fig:cifar_randomsignflip_2}
        \includegraphics[width=0.32\textwidth]{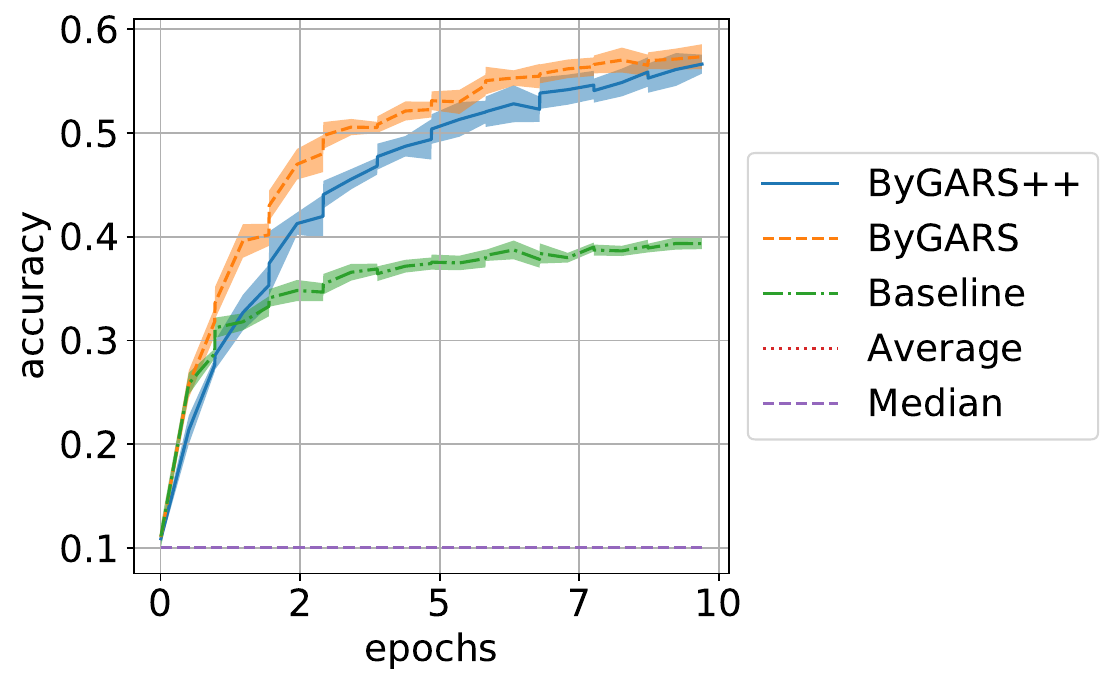}
    }
    \subfloat[]{
        \label{fig:cifar_ofom_2}
        \includegraphics[width=0.32\textwidth]{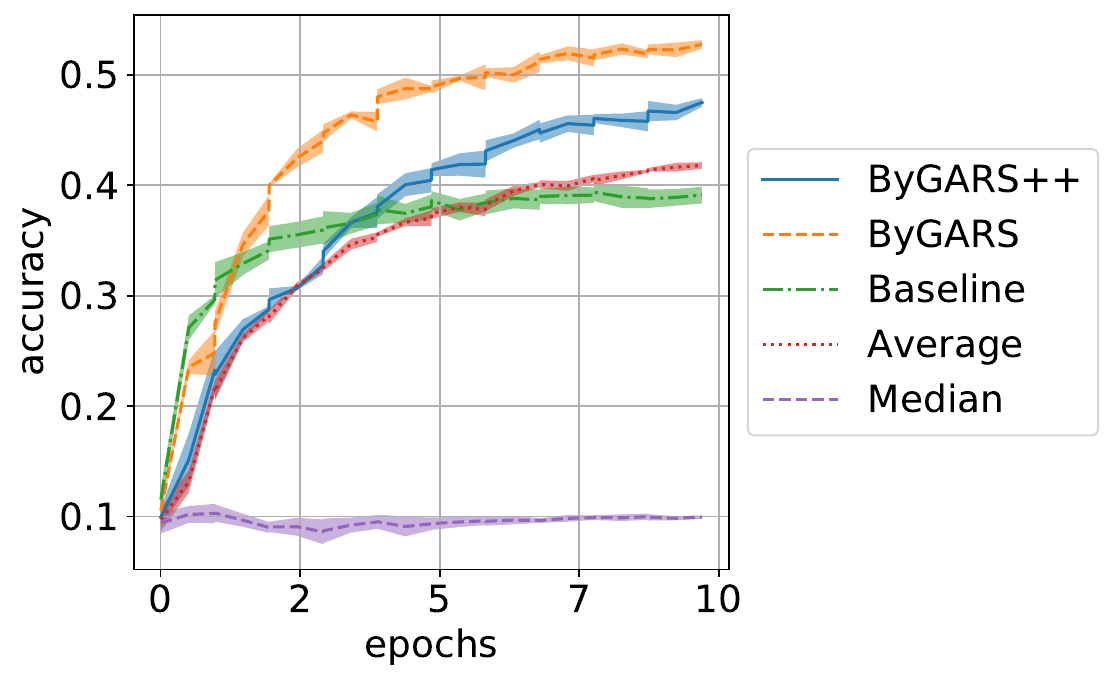}
    }
    \subfloat[]{
        \label{fig:cifar_lie_4}
        \includegraphics[width=0.32\textwidth]{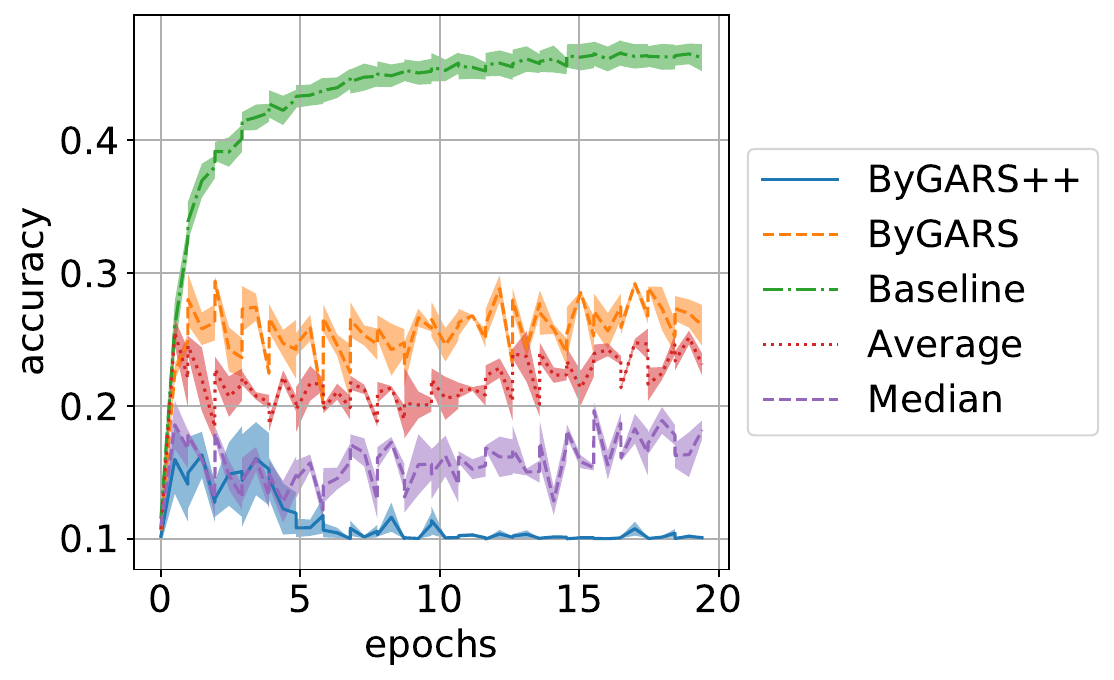}
    }
    \caption{Top-1 accuracy for models trained on CIFAR-10 data under (\ref{fig:cifar_randomsignflip_5}) 3 Random Sign flip adversaries, (\ref{fig:cifar_randomsignflip_2}) 6 Random Sign flip adversaries, (\ref{fig:cifar_ofom_5}) 3 OFOM attackers, (\ref{fig:cifar_ofom_2}) 6 OFOM attackers, (\ref{fig:cifar_lie_5}) 3 LIE attackers, and (\ref{fig:cifar_lie_4}) 4 LIE attackers. ByGARS, ByGARSS++ defend well against random sign flip, and OFOM attackers in all scenarios. ByGARS performs well against LIE with less than 0.5 fraction of adversaries, but suffers some degradation when fraction of adversaries is 0.5. Against LIE, ByGARS++ fails miserably when the fraction of adversaries is 0.5.}
    \label{fig:cifarplots_3}
\end{figure*}

\begin{figure*}[!tbh]
    \centering
    \subfloat[]{
        \label{fig:mnist_const_5}
        \includegraphics[width=0.32\textwidth]{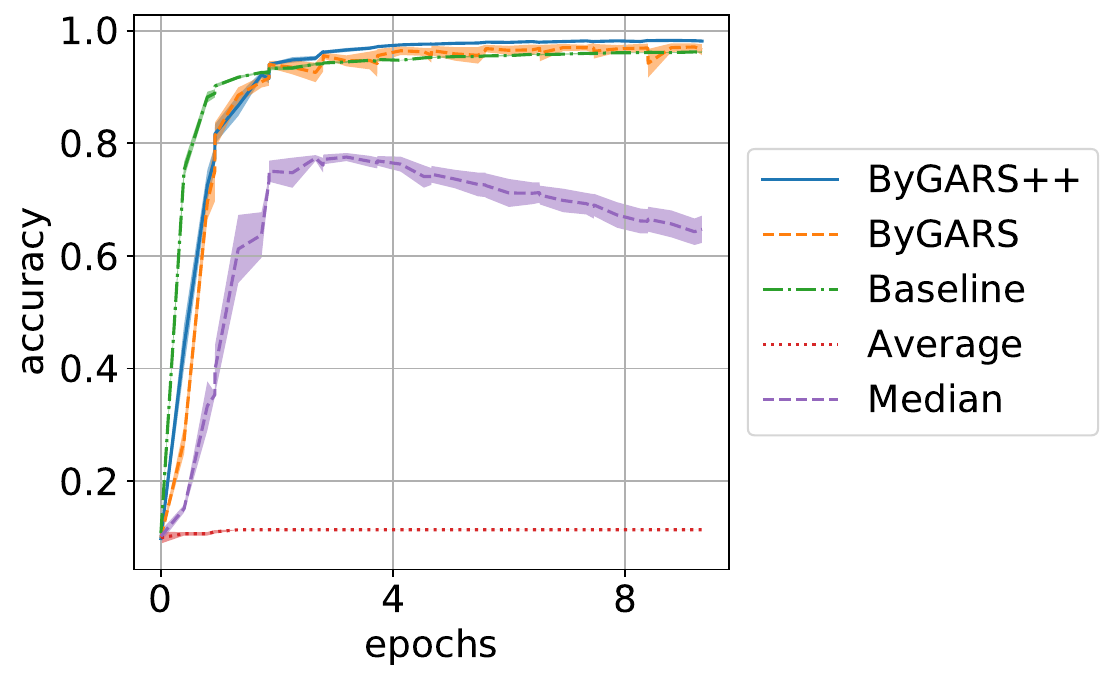}
    }
    \subfloat[]{
        \label{fig:mnist_labelflip_5}
        \includegraphics[width=0.32\textwidth]{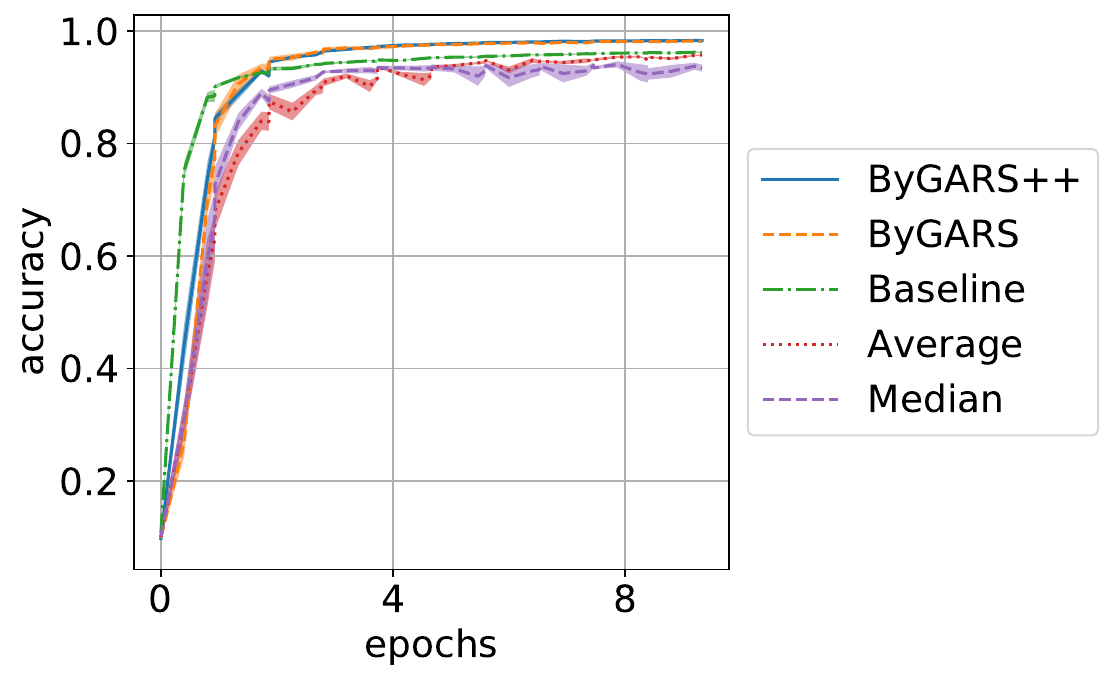}
    }
    \subfloat[]{
        \label{fig:mnist_lie_5}
        \includegraphics[width=0.32\textwidth]{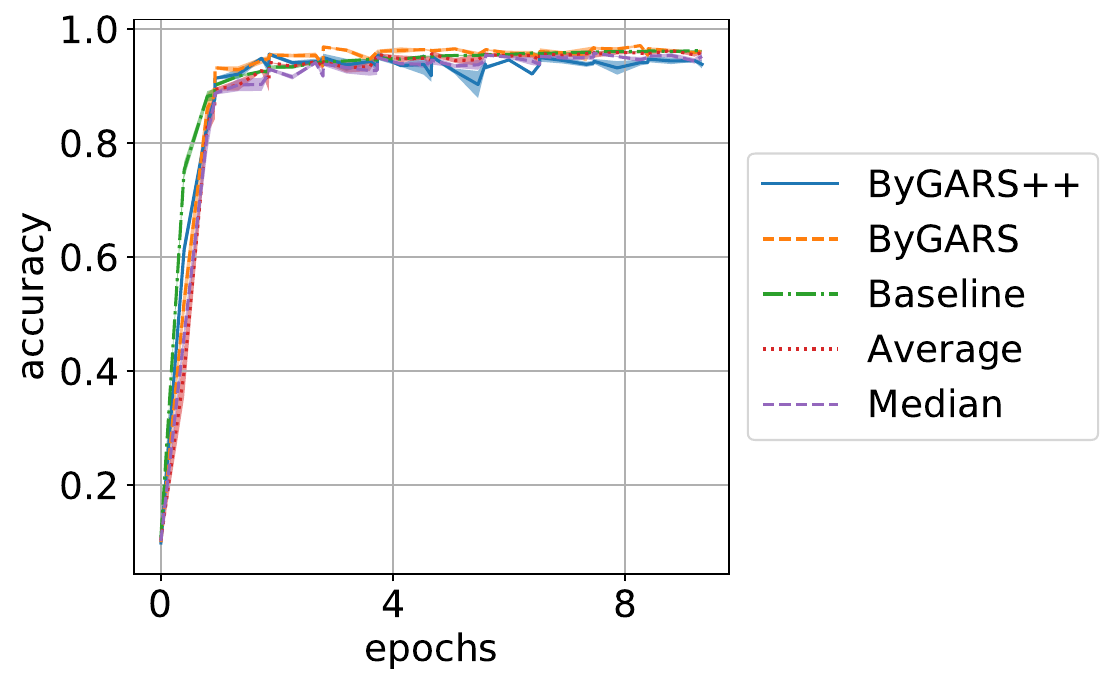}
    }\\
    \subfloat[]{
        \label{fig:mnist_const_2}
        \includegraphics[width=0.32\textwidth]{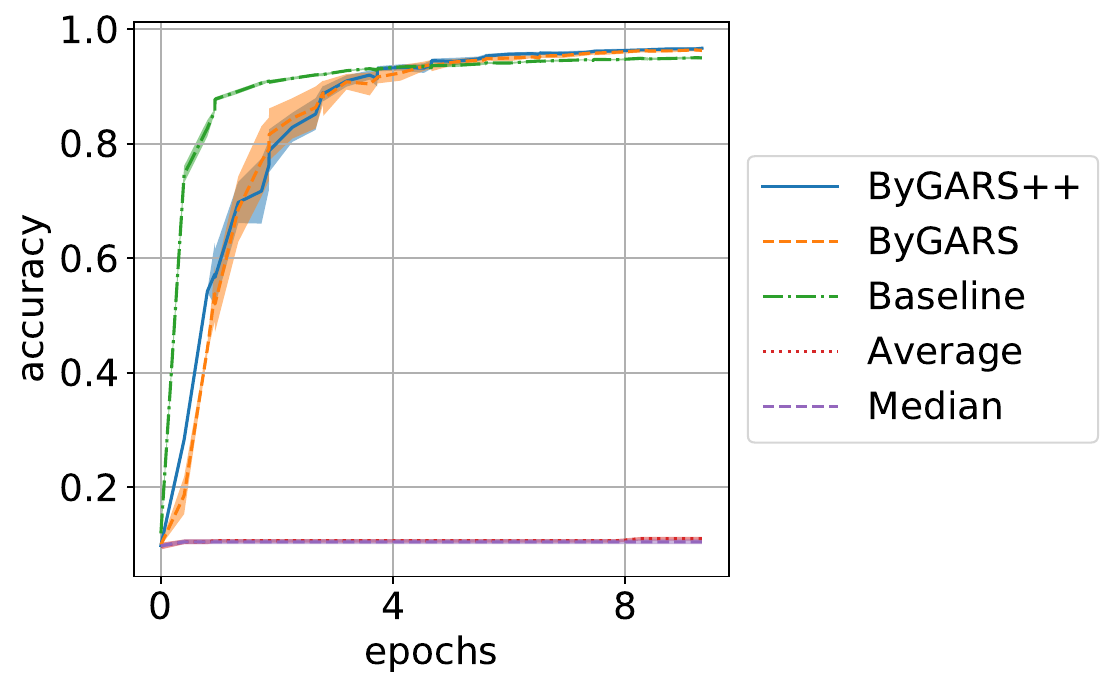}
    }
    \subfloat[]{
        \label{fig:mnist_labelflip_2}
        \includegraphics[width=0.32\textwidth]{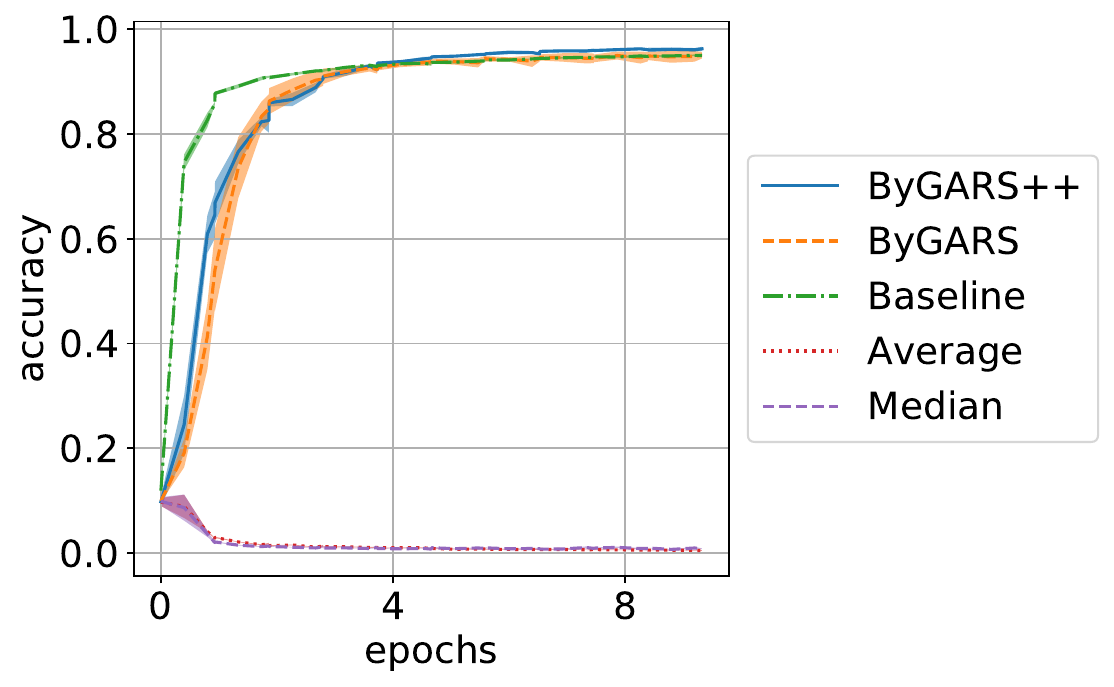}
    }
    \subfloat[]{
        \label{fig:mnist_lie_4}
        \includegraphics[width=0.32\textwidth]{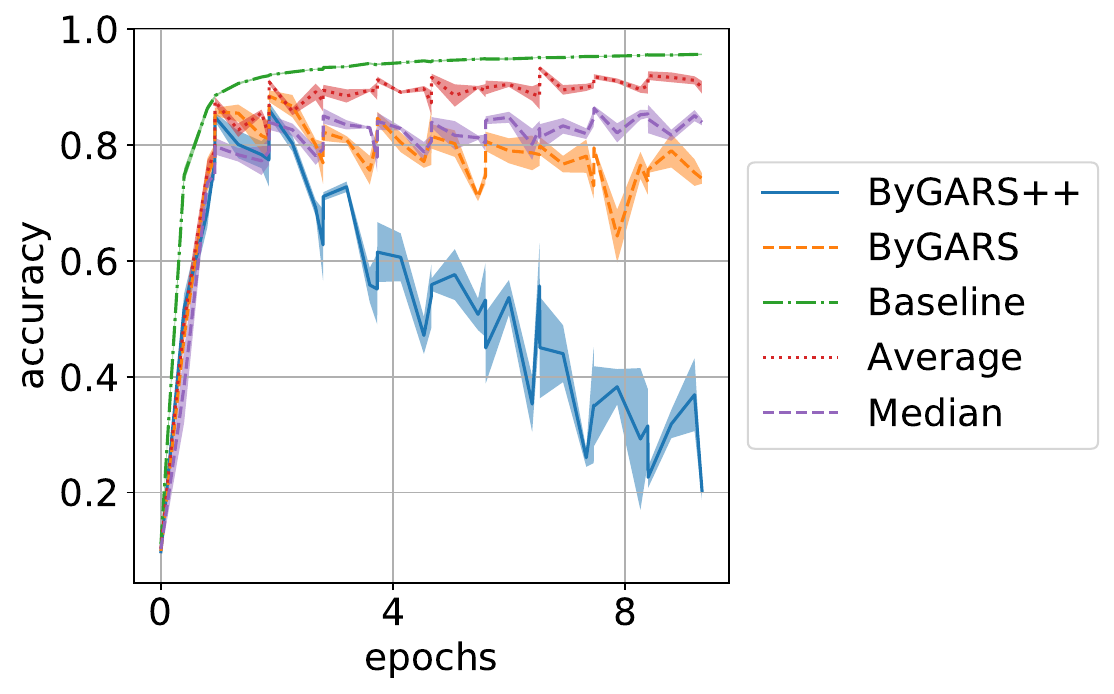}
    }

    \caption{Top-1 accuracy for models trained on MNIST-10 data under (\ref{fig:mnist_const_5}) 3 Constant attackers, (\ref{fig:mnist_const_2}) 6 Constant attackers , (\ref{fig:mnist_labelflip_5}) 3 Label flip attackers, (\ref{fig:mnist_labelflip_2}) 6 Label flip attackers, (\ref{fig:mnist_lie_5}) 3 LIE attackers, and (\ref{fig:mnist_lie_4}) 4 LIE attackers. Similar to the case of CIFAR, both algorithms defend label flip and constant attacks. Against LIE, both algorithms defend well against <0.5 fraction of adversaries. When fraction of adversaries is 0.5, ByGARS performs reasonably well, where as ByGARS++ fails again.}
    \label{fig:mnistplots}
\end{figure*}

\subsection{Discussion} 
\label{discussion}

An advantage of our proposed methods is that, for a given dataset and model, we used the same learning rate, learning rate decay for all types of attacks, with the only difference being the meta learning rate and meta learning rate decay schedules for ByGARS, and ByGARS++. Also, from Fig~\ref{fig:cifar_noattack} it is evident that there is no trade-off in employing our algorithm in the case of \textit{No Attack}. This shows that our proposed algorithms can serve a much general purpose in distributed learning applications.

We can observe from Figures~(\ref{fig:cifarmain},\ref{fig:cifarplots_2},\ref{fig:cifarplots_3},\ref{fig:mnistplots}) that both ByGARS and ByGARS++ achieve Byzantine robustness against most of the threat models used under varying number of adversaries, with the exception of LIE attack. As expected, median fails to defend when the fraction of adversaries is $>0.5$ under all attacks. ByGARS, ByGARS++ on the other hand, do not see a lot of degradation wrt \textit{Baseline} when the fraction of adversaries is $>0.5$.  

Note that, LIE attack \cite{baruch2019little} was only defined for fraction of adversaries $\leq 0.5$, so we evaluated our algorithms against LIE attack with 3 attackers, and 4 attackers. We observe that, ByGARS is robust to 3 attackers for both CIFAR-10 and MNIST datasets, and suffers some degradation in performance when there are 4 attackers. ByGARS++, on the other hand, performs well against 3 attackers on MNIST data, but suffers some degradation on CIFAR-10. In the case of 4 attackers, ByGARS++ fails to defend against the attack. See Fig.~\ref{fig:synth} for experiments on a synthetic dataset (strongly convex objective) and we can see that both ByGARS, ByGARS++ successfully defend 4 \textit{LIE} adversaries. It is interesting to observe that \textit{Median, Average} (no defense) perform reasonably well against LIE with 3 attackers. This was also observed in \cite{baruch2019little} and the authors point that LIE is crafted to break defenses like Krum, Bulyan, etc. The authors explain that using plain averaging, the small noise added to the gradients gets averaged out and the impact on the aggregated gradient is minimal. Perhaps, ByGARS performs well against LIE due to the reputation score based aggregation (weighted averaging) as it was observed that the reputation scores were almost equal for all the workers in this case. 

It is important to note that, one would expect the \textit{Baseline} to be the best in all scenarios. However, we point out that by using the reputation scores, we are directly affecting the step size of each update performed, and hence it is not surprising to observe that our algorithms perform better than the \textit{Baseline} under \textit{No Attack} or weaker adversary models such as Sign Flip. Please refer to Appendix~\ref{sec:appendix_simulations} for more experiments on a synthetic dataset, ablations on the number of meta iterations, size of the auxiliary dataset, and the plots for reputation scores.

Note that the robustness of our algorithm comes from the fact that we did not design the defense based on any particular criteria such as norm, or majority ideas. We devised the algorithm in order for the optimization to find a descent direction, and hence the superior performance across a range of attacks. However, we observe that our algorithm is not robust to LIE attacks in one scenario (fraction of adversaries=0.5) which is currently the state of the art attack in Byzantine setting. We would like to leave investigating this issue as part of our future work. 


\section{Conclusion}

We devise a novel, Byzantine resilient, stochastic gradient aggregation algorithm for distributed machine learning with arbitrary number of adversarial workers. This is achieved by exploiting a small auxiliary dataset to compute a reputation score for every worker, and the scores are used to aggregate the workers' gradients. 
We showed that under reasonable assumptions, ByGARS++ converges to the optimal solution using a result from two timescale stochastic approximation theory \cite{tadic2004almost}. Through simulations, we showed that the proposed algorithms exhibit remarkable robustness property even for non-convex problems under a wide range of Byzantine attacks. Although ByGARS and ByGARS++ are developed for this specific setting of Byzantine adversaries, we believe that these algorithms serve a much general purpose. This algorithm can be modified to train models in other cases such as learning from heterogeneous datasets, learning under privacy constraints, other adversarial settings (such as injecting different noises to individual dimensions of the gradient), and poisoned data attacks. We leave such analyses for a future work.

\section*{Acknowledgements}
The authors thank ARPA-E NEXTCAR program and ARO Grant W911NF1920256 for supporting the research. Results presented in this poster were obtained using the Chameleon testbed supported by the National Science Foundation.

\small

\bibliographystyle{unsrt}  
\bibliography{ref}

\appendix

\section{Proof of Theorem 1}


In order to establish this result, we first define the filtration and then prove a lemma. Let $(\Omega,\mathcal F,\mathbb{P})$ be a standard probability space and let $\mathcal F_t$ be the $\sigma$-algebra generated by all the randomness realized up to time $t$:
\begin{align*}
   \F_0 =& \sigma\{\w_0,\bq_0\},\\
   \F_t =& \sigma\{\w_0,h_{1,0},\ldots,h_{m,0},\bq_0,\ldots,\w_{t-1},h_{1,t-1},\ldots,
   h_{m,t-1},\bq_{t-1},\w_t,\bq_t\}.
\end{align*}
It is easy to see that $\mathcal F_t\subset\mathcal F_{t+1}$, and thus, $\{\mathcal F_t\}_{t\in\mathbb{N}}$ is a filtration. We can now establish the following result.
\begin{lemma}
The following holds:
\begin{subequations}\label{eqn:qtkappa}
\begin{align}
\E[\bq_{t+1}|\mathcal F_t] &= (1-\alpha_t)\bq_t + \alpha_t \norm{\nabla F(\w_t)}^2 \bk \label{eqn:qtkappai}\\
\E[\kappa_i q_{t+1,i}|\mathcal F_t] &= (1-\alpha_t)\bk_i q_{t,i} + \alpha_t \norm{\nabla F(\w_t)}^2 \kappa_i^2 \label{eqn:qtkappaii}\\
\E[\bk^T\bq_{t+1}|\mathcal F_t] &= (1-\alpha_t)\bk^T \bq_t + \alpha_t \norm{\nabla F(\w_t)}^2 \|\bk\|^2\label{eqn:qtkappaiii}
\end{align}
\end{subequations}
\end{lemma}
\begin{proof}
We know that $\E[H_t|\F_t] = \bk \nabla F(\w_t)^T $ and $\E[\nabla L_t(\w_t)|\F_t] = \nabla F(\w_t)$. Using these relationships in the definition of $\bq_{t+1}$ in ByGARS++, we get
\begin{align*}
    \bq_{t+1} &= (1-\alpha_t)\bq_t + \alpha_t H_t \nabla L_t(\w_{t})\\
    \E [\bq_{t+1} | \mathcal F_t] &= (1-\alpha_t)\bq_t + \alpha_t \E [H_t \nabla L_t(\w_t) | \w_t, \bq_t] \\
    &= (1-\alpha_t)\bq_t + \alpha_t \norm{\nabla F(\w_t)}^2\bk.
\end{align*}
This establishes the equality in \eqref{eqn:qtkappai}, \eqref{eqn:qtkappaii}, and \eqref{eqn:qtkappaiii}. 
\end{proof}
An immediate corollary to the above result is as follows. Define $\beta_{s,t} = \prod_{k=s+1}^t(1-\alpha_k)>0$ with the convention that $\beta_{t,t} = 1$. We can rewrite $\E[\bk^T\bq_{t+1}|\F_t]$ as
\begin{align}
    \E[\bk^T\bq_{t+1}|\F_t] = \sum_{s=0}^t\beta_{s,t}\alpha_s\|\bk\|^2\|\nabla F(\w_s)\|^2,\label{eqn:kappaq}
\end{align}
where we used the fact that $\bq_0 = 0$. 

We now proceed to establish Theorem 1, which follows from the law of iterated expectation and the principle of mathematical induction. Consider the stochastic process $\bk^T \bq_{t}$ in \eqref{eqn:qtkappaiii}. At time $t =0$, $\bq_0 = 0$. Thus, we have $\E[\bk^T\bq_1|\F_0] >0$ since we assumed that there is at least one $\kappa_i\neq 0$ and $\nabla F(\w_0)\neq 0$. Thus, we conclude that $\E[\bk^T\bq_1] >0$ by the law of iterated expectation. The same argument establishes the fact that $\E[\bk^T \bq_t]>0$ for all $t\geq 1$, since the right side of \eqref{eqn:qtkappaiii} is a convex combination of terms that are positive in expectation (until $\w_t = \w_*$). This can also be deduced from \eqref{eqn:kappaq} by taking the expectation of the right side and noting that it is positive.

We now apply the same argument to $\mathbb{E}[\dott{\nabla F(\w_{t}), H_t^T\bq_{t}}|\F_t]$. Clearly,
\begin{align}
    \E [\dott{\nabla F(\w_{t}) , H_t^T\bq_{t}}|\F_t] &= \dott{\nabla F(\w_{t}), \E[H_t^T\bq_t|\F_t]}\\
    &= \dott{\nabla F(\w_{t}), \nabla F(\w_{t})\bk^T\bq_{t}}\nonumber\\
    &= \bk^T\bq_t \norm{\nabla F(\w_t)}^2\label{eqn:bknabla}
\end{align}
Coupling the above expression with \eqref{eqn:kappaq}, we get 
\begin{align*}
    &\E [\dott{\nabla F(\w_{t}) , H_t^T\bq_{t}}] =\sum_{s=0}^{t-1}\alpha_s\beta_{s,t-1}\|\bk\|^2\E[\norm{\nabla F(\w_s)}^2\|\nabla F(\w_t)\|^2]> 0,
\end{align*}
since both terms within the expectation are positive random numbers unless $\w_t = \w_*$. This completes the proof of the Theorem 1.

\section{Simulations}
\label{sec:appendix_simulations}
\subsection{Setup}

The simulations are implemented in Pytorch. Although the algorithm is aimed at robustifying distrbuted machine learning, in the implementation we only use one node while processing the workers sequentially so that the results can be reproduced even on a single machine. We train the models on Nvidia RTX 6000 GPU.

\subsection{Hyperparameters} We used the following hyper-parameters for different configurations. In this section, the step size parameter $\gamma_t$ used to update $\w_t$ is called as learning rate, and the step size parameter $\alpha_t$ used to update $\bq_t$ is called as meta-learning rate. We scheduled the learning rate ($\gamma_t$) to decay as: $\gamma_t = \gamma_0 \times \frac{1}{1 + \beta t}$, and meta learning rate ($\alpha_t$) as : $\alpha_t = \alpha_0 \times \frac{1}{1+\beta_m t^{0.9}}$. Note that we only change the meta update parameters for the two algorithms. We used the following hyperparameters

\begin{enumerate}
    \item CIFAR10: $\gamma_0$= 0.2, $\beta$= 0.9, 
    \begin{itemize}
        \item ByGARS: $\alpha_0$= 0.2, $\beta_m= 0.5$
        \item ByGARS++: $\alpha_0$= 0.001, $\beta_m= 0.1$
    \end{itemize}
    \item MNIST: $\gamma_0$= 0.05, $\beta$= 0.5,
    \begin{itemize}
        \item ByGARS: $\alpha_0$= 0.05, $\beta_m= 0.5$
        \item ByGARS++: $\alpha_0$= 0.001, $\beta_m= 0.2$
    \end{itemize}
\end{enumerate}

Since our objective is to illustrate the effectiveness of the proposed algorithms against the adversaries, we did not perform hyperparameter tuning to obtain best test accuracy. Instead, we show that for appropriately chosen hyperparameters, the proposed algorithms converge and are Byzantine resilient. Unless otherwise mentioned, all the results mentioned use an auxiliary dataset of size 250, and $k=3$ meta iterations (for ByGARS).

\subsection{Ablation study on auxiliary dataset size}

The key to the effectiveness of ByGARS and ByGARS++ is the availability of an auxiliary dataset that is drawn from the same distribution as the testing and training data. We study the dependence of the algorithms on the auxiliary dataset by repeating the experiments with different sizes of the auxiliary dataset. From Fig~(\ref{fig:aux_size_bygars}, \ref{fig:aux_size_bygars_plus}) we can observe that a very small amount of auxiliary dataset is sufficient to face any number of adversarial workers, and the performance increases with increase in the size of the auxiliary dataset. However, we observe that the algorithm is quite robust to the size of the auxiliary dataset once a sufficient size is reached. Note that the default size used in all other experiments is 250.

\subsection{Ablation study on number of meta iterations} In order to study the dependence of ByGARS on the number of meta iterations, we repeated the experiments with number of meta iterations ranging from 1 to 4 (we don't do this for ByGARS++ since there is only one meta update). We show the results in Fig~\ref{fig:meta_itr_bygars}. We can observe that, higher the number of meta iterations, the better is the test accuracy. However, we did not study if with increasing number of meta-iterations, we over fit to the auxiliary data set and leave it for future work. Note that in the rest of the experiments, we used 3 meta iterations.

\begin{figure*}[t]
    \centering
    \subfloat[]{
        \label{fig:aux_size_bygars}
        \includegraphics[width=0.33\textwidth]{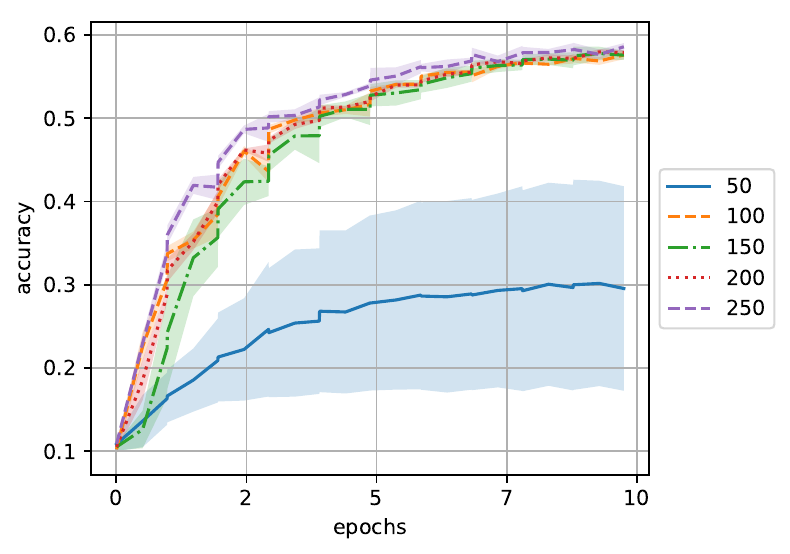}
    }
    \subfloat[]{
        \label{fig:aux_size_bygars_plus}
        \includegraphics[width=0.33\textwidth]{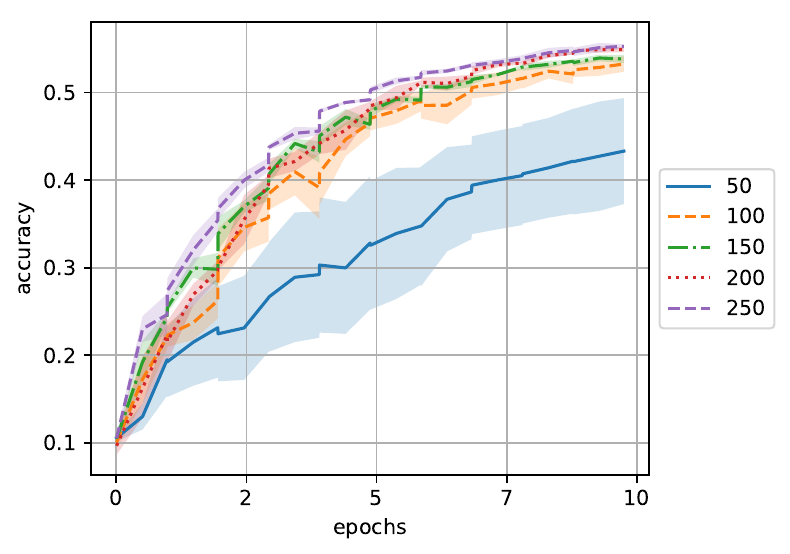}
    }
    \subfloat[]{
        \label{fig:meta_itr_bygars}
        \includegraphics[width=0.33\textwidth]{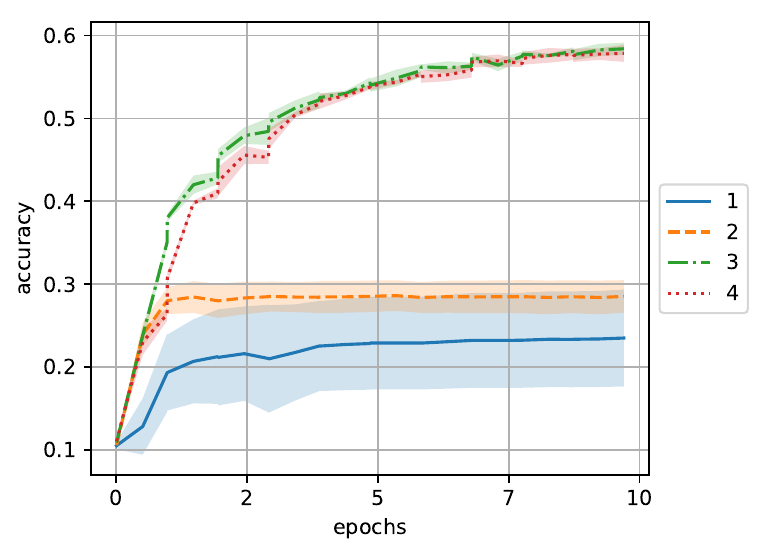}
    }\\
    \caption{\label{fig:ablation} (\ref{fig:aux_size_bygars}) Ablation on size of auxiliary dataset for ByGARS under 3 Label Flip adversaries; (\ref{fig:aux_size_bygars_plus}) Ablation on size of auxiliary dataset for ByGARS++ under 3 Label Flip adversaries; (\ref{fig:meta_itr_bygars}) Ablation on number of meta iterations for ByGARS under 3 Label Flip adversaries. }
\end{figure*}

\begin{figure*}[t]
    \centering
    \subfloat[]{
        \label{fig:synth_mixed}
        \includegraphics[width=0.33\textwidth]{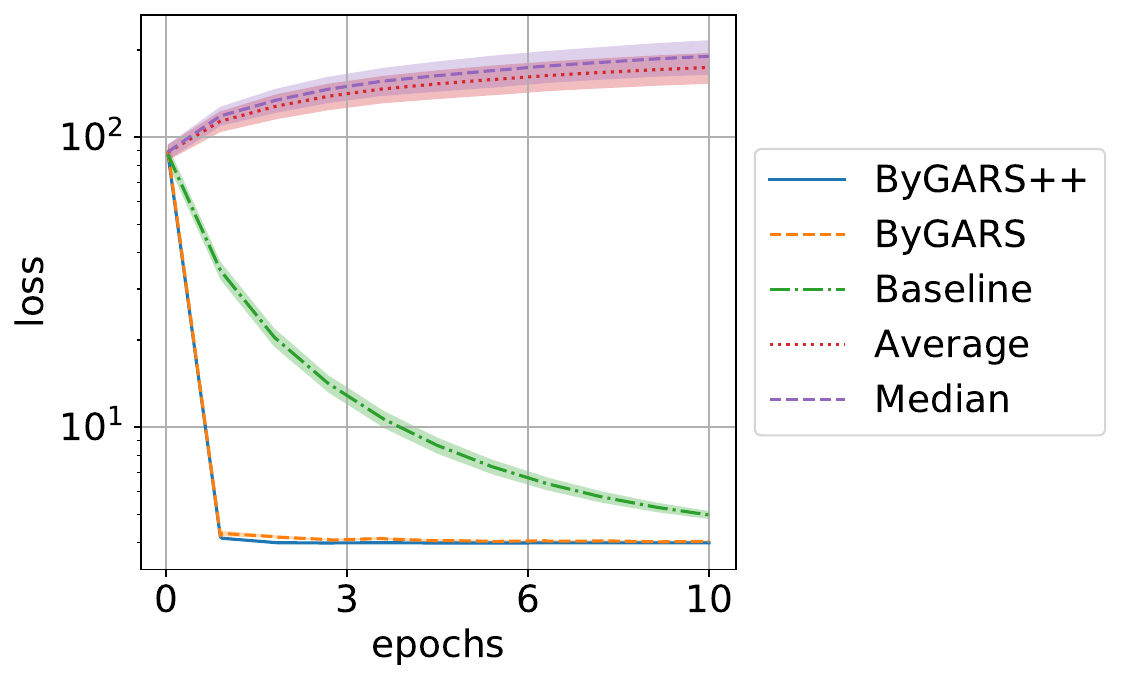}
    }
    \subfloat[]{
        \label{fig:synth_const_2}
        \includegraphics[width=0.33\textwidth]{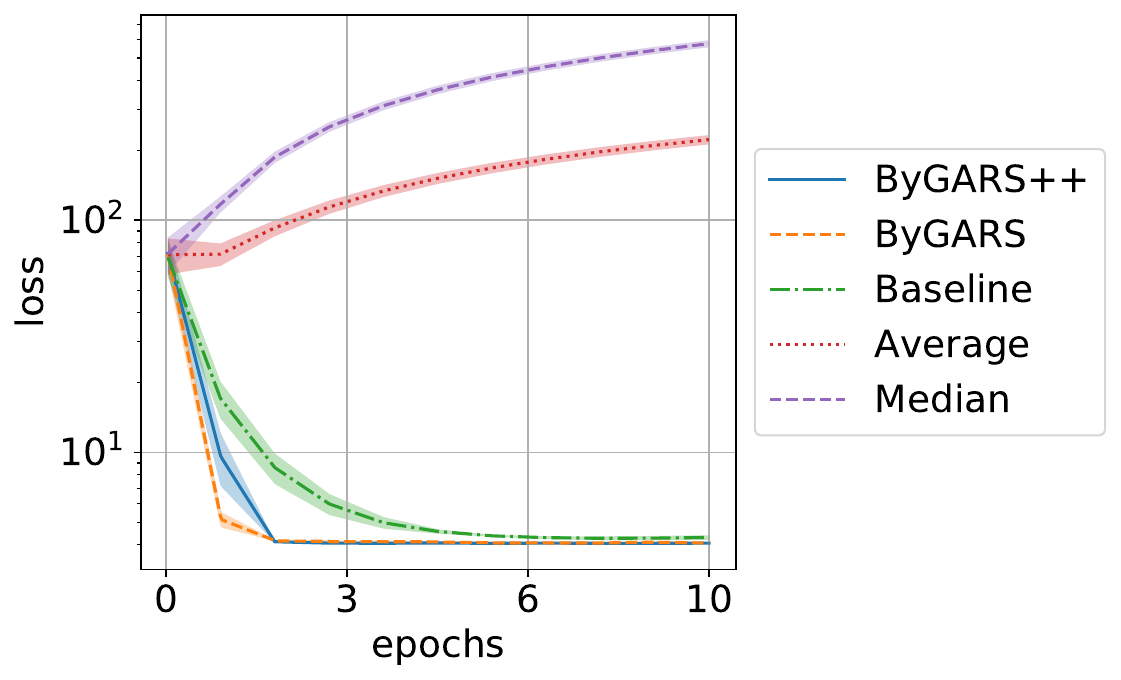}
    }
    \subfloat[]{
        \label{fig:synth_lie_4}
        \includegraphics[width=0.33\textwidth]{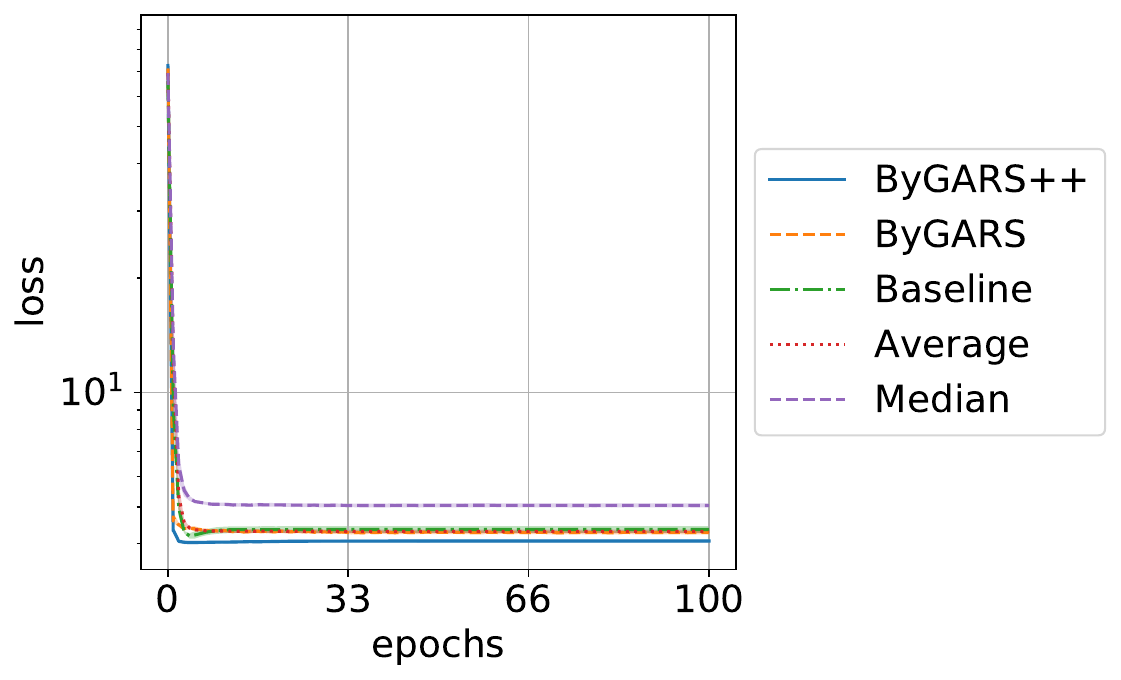}
    }
    \caption{\label{fig:synth}  (\ref{fig:synth_mixed}) MSE loss of the synthetic dataset under \textit{Mixed Attack}; (\ref{fig:synth_const_2}) MSE loss of the synthetic dataset under 6 \textit{Constant Attack} adversaries; (\ref{fig:synth_lie_4}) MSE loss of the synthetic dataset under 4 \textit{LIE} adversaries.}
\end{figure*}

\subsection{Synthetic Datset}
In addition to CIFAR-10, and MNIST, we also evaluated our algorithm on a linear regression task using a synthetic datset. In this section we will provide details of the construction of synthetic dataset, and the results on \textit{Mixed Attack}. 

Given input dimension $d$, we draw a gaussian vector in $\mathbb{R}^{d}$ with non-zero mean and identity covariance and denote it as $\theta^*$. Next, given the number of datapoints $N$, we generate $N$ random vectors in $\mathbb{R}^d$. For each data point, we assign $y=x^T\theta^*+\epsilon$ where $x$ is the datapoint, and $\epsilon$ follows a zero mean Gaussian distribution. The objective is to find $\theta$ such that the mean square error $\norm{Y-X^T\theta}^2$ is minimized. In the simulations, we used $d=20, N=10,000$ of which 2000 were used for testing, and the remaining 8000 is split between training data and the auxiliary data (size=250). We provide results for \textit{Mixed Attack, Constant Attack} and \textit{LIE} on this synthetic task in Fig~\ref{fig:synth}, where we evaluate using the Mean Squared Error (MSE) on the test set. We can observe that both the proposed algorithms are robust to all the attacks. We ran \textit{LIE} attack for 100 epochs to verify that performance doesn't drop in the middle of training (which was observed with the case of ByGARS++).

\subsection{Reputation Scores}

Fig.~\ref{fig:q_weights} shows how the reputation scores evolve over time for different scenarios under mixed attack. It can be observed from Fig.~(\ref{fig:q_mnist_bygars}, ~\ref{fig:q_synth_bygars}) that the reputation scores for ByGARS reach an equilibrium value and stays there throughout the training. The reputation scores for ByGARS++ on the other hand converge towards zero as the training progresses, see Fig.~(\ref{fig:q_mnist_bygarsplus},\ref{fig:q_synth_bygarsplus}). Note that in \textit{Mixed Attack}, only worker 0 is benign, while the rest are adversarial, and we can observe from Fig.~\ref{fig:q_weights} that the reputation scores for the adversaries are either close to zero or negative depending on the type of the adversary.

\begin{figure*}[b]
    \centering
    \subfloat[]{
        \label{fig:q_mnist_bygars}
        \includegraphics[width=0.45\textwidth]{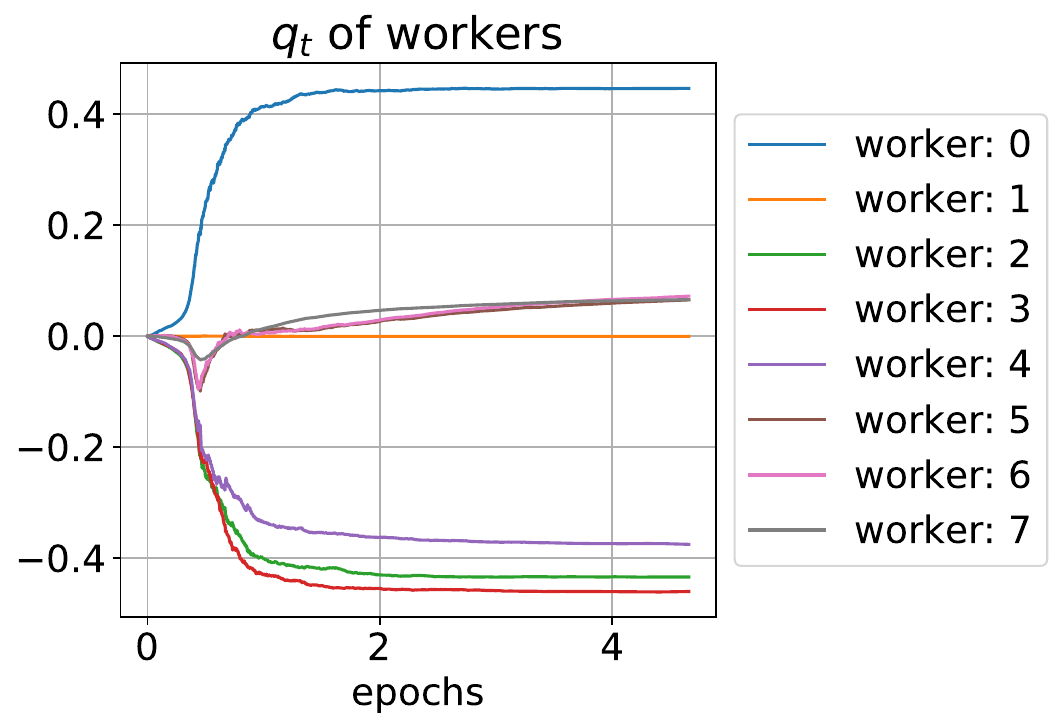}
    }
    \subfloat[]{
        \label{fig:q_mnist_bygarsplus}
        \includegraphics[width=0.45\textwidth]{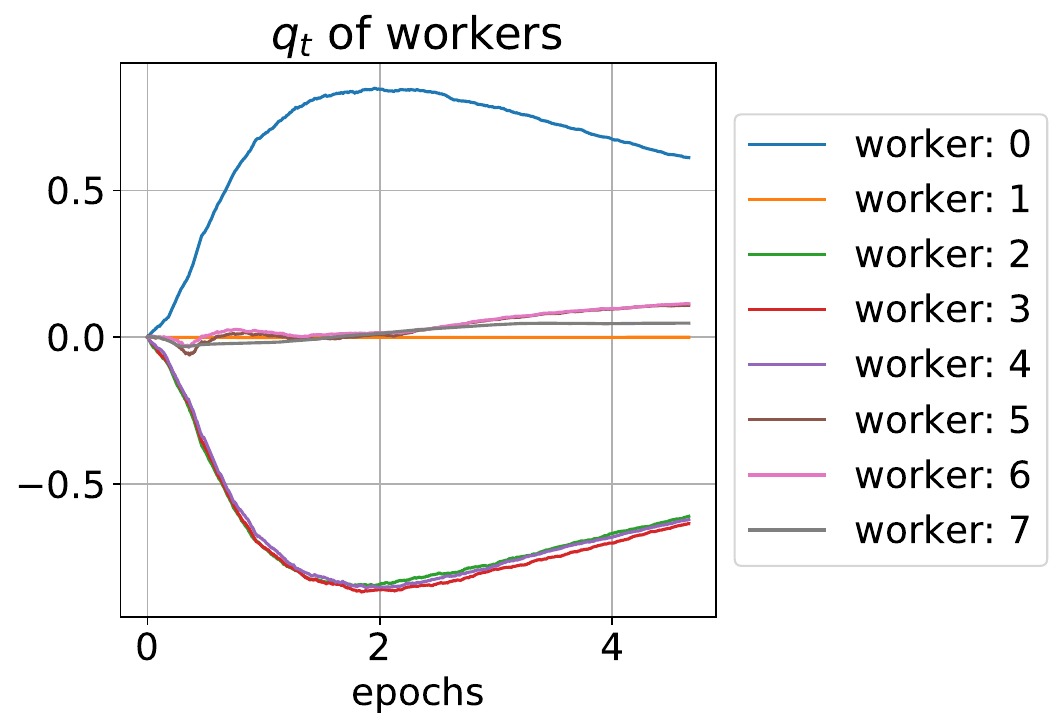}
    }\\
    \subfloat[]{
        \label{fig:q_synth_bygars}
        \includegraphics[width=0.45\textwidth]{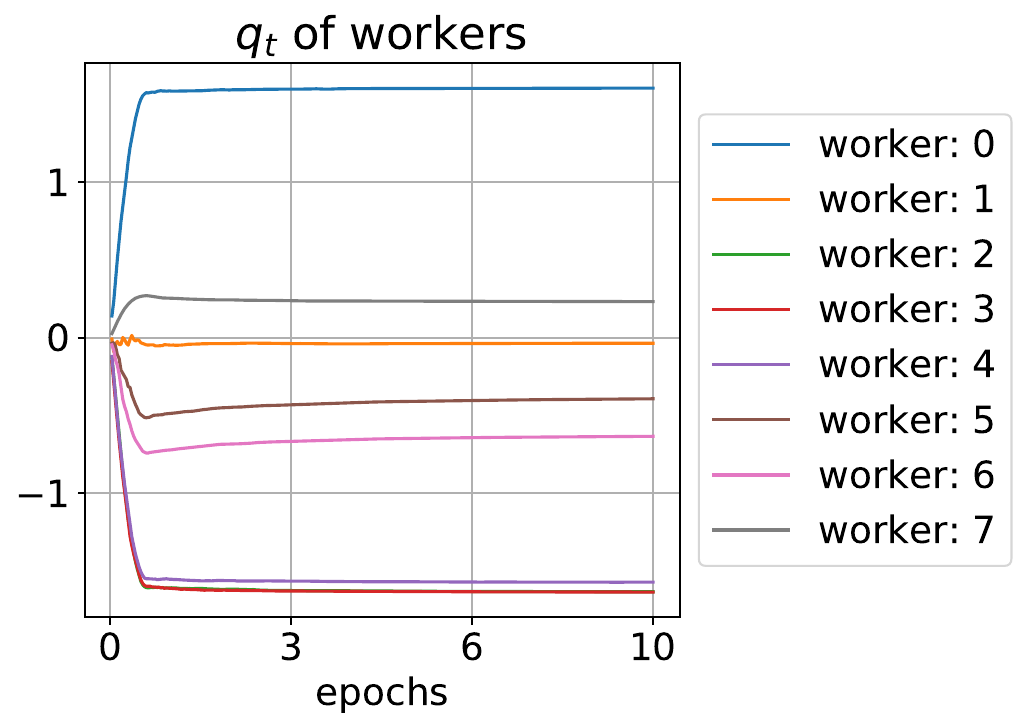}
    }
    \subfloat[]{
        \label{fig:q_synth_bygarsplus}
        \includegraphics[width=0.45\textwidth]{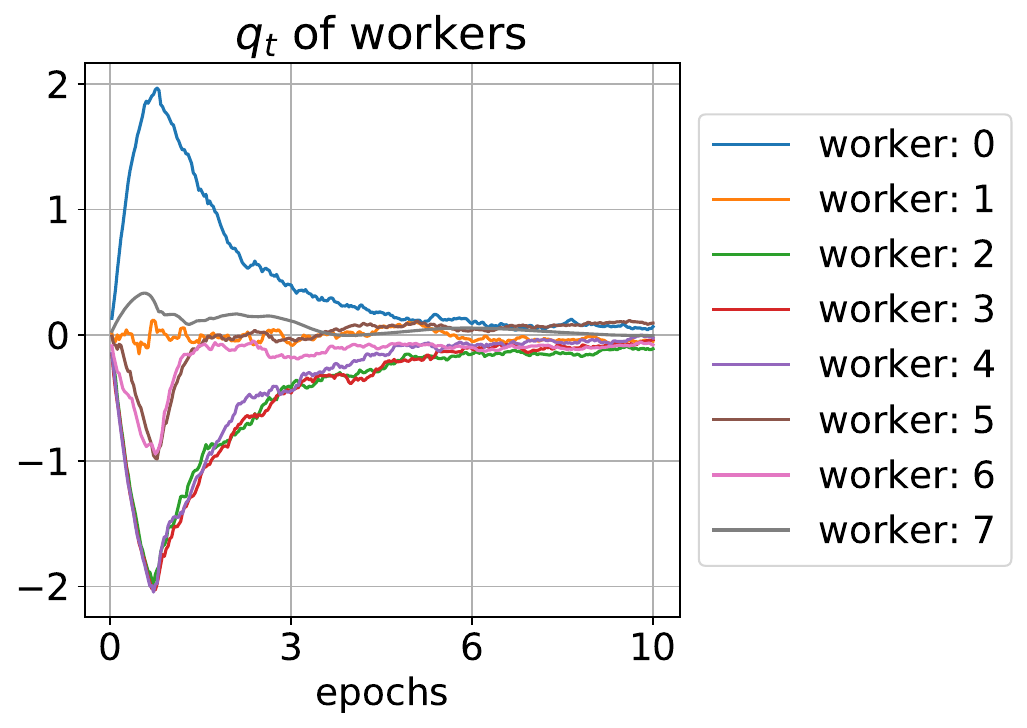}
    }\\
    \caption{\label{fig:q_weights} This figure shows the reputation scores for ByGARS, ByGARS++ under different scenarios. Attack carried out is mixed attack in all cases. (\ref{fig:q_mnist_bygars}) ByGARS on MNIST; (\ref{fig:q_mnist_bygarsplus}) ByGARS++ on MNIST;
    (\ref{fig:q_synth_bygars}) ByGARS on synth; (\ref{fig:q_synth_bygarsplus}) ByGARS++ on synth;}
\end{figure*}

\end{document}